\documentclass{new_tlp}
\usepackage{mathptmx}

\usepackage{times}
\usepackage{helvet}
\usepackage{courier}
\usepackage{url}
\usepackage{amsfonts,amsmath}
\usepackage{latexsym}
\usepackage{color}
\usepackage{url}
\usepackage{float}
\usepackage{multicol,multirow,caption}
\usepackage{listings}
\usepackage{xspace,epsfig}
\usepackage{color,xcolor}

\newtheorem{theorem}{Theorem}
\newtheorem{corollary}{Corollary}

\def\lar{\leftarrow}
\def\ba{\begin{array}}
\def\ea{\end{array}}
\def\bi{\begin{itemize}}
\def\ei{\end{itemize}}
\def\be{\begin{enumerate}}
\def\ee{\end{enumerate}}
\def\beq{\begin{equation}}
\def\eeq#1{\label{#1}\end{equation}}
\def\no{\ii{not}}
\def\ii#1{\hbox{\sl #1\/}}

\def\eqs{\,{=}\,}

\def\leqs{\,{\leq}\,}
\def\geqs{\,{\geq}\,}
\def\lts{\,{<}\,}
\def\gts{\,{>}\,}
\def\ins{\,{\in}\,}
\def\cups{\,{\cup}\,}

\def\clingo{{\sc Clingo}\,}
\def\Reg{\textbf{\textit{Reg}}\,}
\def\Regs{\textbf{\textit{Reg*}}\,}

\def\3ncdc{\textsc{3D-nCDC-ASP}\,}
\def\2ncdc{\textsc{nCDC-ASP}\,}

\usepackage{todonotes}

  \title[Theory and Practice of Logic Programming]
        {Reasoning about Cardinal Directions between\\3-Dimensional Extended Objects\\using Answer Set Programming}

  \author[Izmirlioglu and Erdem]{
         Yusuf Izmirlioglu and Esra Erdem\\
         Sabanci University, Faculty of Engineering and Natural Sciences, 34956 Istanbul, Turkey\\
         \{yizmirlioglu,esra.erdem\}@sabanciuniv.edu}

\begin{document}

\label{firstpage}

\maketitle

  \begin{abstract}
We propose a novel formal framework (called \3ncdc) to represent and reason about cardinal directions between extended objects in 3-dimensional (3D) space, using Answer Set Programming~(ASP). \3ncdc extends Cardinal Directional Calculus (CDC) with a new type of default constraints, and \2ncdc to 3D. \3ncdc provides a flexible platform offering different types of reasoning:  Nonmonotonic reasoning with defaults, checking consistency of a set of constraints on 3D cardinal directions between objects, explaining inconsistencies, and inferring missing CDC relations. We prove the soundness of \3ncdc, and illustrate its usefulness with applications. This paper is under consideration for acceptance in TPLP.
  \end{abstract}




\section{Introduction}

Qualitative spatial reasoning studies representation and reasoning with different aspects of space, such as direction, distance, size 
 using parts of natural language 
 rather than quantitative data.
Qualitative models are useful in contexts where quantitative data is not available due to incomplete knowledge or uncertainty.
Examples are exploration of an unknown territory such as disaster rescue, marine habitat discovery and underwater archeology.

Qualitative reasoning is also relevant for contexts with complete information and quantitative data because human agents tend to express spatial relation or configuration by means of qualitative terms for the sake of sociable and convenient communication. For instance, while designing a building, it is more intuitive and understandable to describe the location of the transformer room as follows: ``The transformer room must be at the rear side of the building, near the electric panel. It should be located on a lower level than the entrance." In this case, a quantitative description may be too complicated or even not possible.

Most qualitative calculi and reasoning mechanisms have been developed for objects in  1-dimensional (1D) or 2-dimensional (2D) space, as described in the surveys~\cite{chen2015survey,dylla2017survey}.
However, in real environments, agents move and explore in all 3 dimensions or deal with complex 3-dimensional (3D) objects.
In this paper, we study representation of and reasoning about qualitative directions in 3D space.  We consider 3D cardinal directions (e.g., to the north and above, the east and below, to the southwest and on the same level) as in the related studies~\cite{chen2007cardinal,hou2016reasoning}, which extend Cardinal Directional Calculus (CDC)~\cite{Goyal1997,SkiadKoub2004,SkiadKoub2005} to 3D space. Different from these studies, instead of blocks (rectangular prism shape objects), we consider 3D objects of arbitrary shapes, that may be disconnected.

In CDC, cardinal directions between objects are represented by formulas called CDC constraints. In our study, to incorporate commonsense knowledge into reasoning, we introduce a new type of constraint (called {\em default 3D constraint}) to represent default relations (e.g., the garage is by default below and to the north of the entrance in a building). We call this extended version of 3D-CDC as {\em 3-dimensional nonmonotonic CDC (3D-nCDC)}.

One of the central problems in 3D CDC literature is the consistency checking of a set of 3D CDC constraints. Informally, this problem is concerned about the existence of a possible configuration of objects with respect to the given CDC constraints. We study consistency checking in 3D-nCDC, and provide a general solution that is not restricted to tractable cases as in related work.
In addition to consistency checking, we consider other forms of reasoning important for various real-world applications: nonmonotonic reasoning, explaining inconsistencies, and inferring missing CDC relations between objects.

We propose a formal framework to represent 3D-nCDC constraints and to reason about them, using the logic programming paradigm Answer Set Programming (ASP)~\cite{MarekT99,Niemelae99,Lifschitz02}, based on the answer set semantics~\cite{GelfondL88,gelfondL91}. For that reason, we call this framework as \3ncdc.  We show the soundness and completeness of \3ncdc, implement it using the ASP language ASP-Core-2~\cite{calimeri2020asp} and the ASP solver \clingo~\cite{GebserKKOSS11}, and show interesting applications in marine exploration using underwater robots, building design and regulation, and evidence-based digital forensics.
Proofs are provided in Appendix.


\section{Related Work}\label{sec:related}

Cardinal directions in 3D have been studied in the literature for blocks, by directly extending CDC to 3D space~\cite{chen2007cardinal,hou2016reasoning},
by utilizing projections of objects into 1D~\cite{pais2000spatial} or 2D~\cite{li2009qualitative}, or
in terms of the 13 relations of Interval Algebra~\cite{Allen83} as in the block algebra~\cite{balbiani1999tractable,BalbianiCC02}. In our study, we understand 3D cardinal directions as in the related studies~\cite{chen2007cardinal,hou2016reasoning}, instead of combinations of lower-dimensional relations that may not be directional. Different from these studies: (i) instead of blocks, we consider 3D objects of arbitrary shapes, that may be disconnected, (ii) to incorporate commonsense knowledge into reasoning, we introduce default 3D constraints to represent default relations. An example that illustrates the strengths of adopting directly a 3D calculus instead of projecting it to lower dimensions is available in~\ref{sec:proj-3d}.

One of the central problems studied in 3D CDC is the consistency checking of a set of 3D CDC constraints. Polynomial time algorithms have been introduced by~\citeN{chen2007cardinal} and~\citeN{hou2016reasoning} for consistency checking in 3D CDC under the condition that constraints are basic (i.e., not disjunctive).
Different from these studies: (iii) we study the consistency checking problem in 3D-nCDC and provide a general solution, but without restricting it to the tractable cases, (iv) we also consider other forms of reasoning important for various real-world applications: nonmonotonic reasoning, explaining inconsistencies, and inferring missing CDC relations between objects, and (v) we propose a formal framework (called \3ncdc) to represent 3D-nCDC constraints and to reason about these constraints, using ASP.

ASP has been applied to solve the consistency checking problem in 1D and 2D qualitative calculi. For instance, \citeN{brenton2016answer} represent Region Connection Calculus with eight base relations (RCC-8)~\cite{cohn1997qualitative}, \citeN{walega2017non} represent RCC-5~\cite{Randell1992ASL}, and \citeN{baryannis2018trajectory} represent Trajectory Calculus~\cite{baryannis2018trajectory} in ASP.
Like \cite{brenton2016answer} and \cite{baryannis2018trajectory}, we utilize the ASP language ASP-Core-2 and the ASP solver \clingo; \cite{walega2017non} utilizes ASPMT language, 
and the SMT solver Z3~\cite{deMoura2008}. Different from these studies: (a) we consider a 3D qualitative calculus, and extend it with new types of default constraints whose semantics is provided by means of the nonmonotonic constructs of ASP. Furthermore, (b) we consider not only consistency checking but also other reasoning problems mentioned above.

\3ncdc extends our earlier work \2ncdc~\cite{izmirlioglu018}, which investigates nonmonotonic CDC in 2D using ASP, to 3D.
We represent 3D cardinal directions between 3D extended objects, perform consistency checks of 3D-nCDC constraints, and generate missing 3D cardinal direction relations between objects.
Our representation of 3D-nCDC constraints is methodologically different, enabling generation of explanations for inconsistencies, and enabling a more general definition of default CDC constraints.

Qualitative directional relations in 3D are used in robotics. For instance, \citeN{zampogiannis2015learning} define six directional relations (i.e., \ii{left}, \ii{right}, \ii{front}, \ii{behind}, \ii{below}, \ii{above}) between point clouds in 3D by utilizing cones, for the purpose of grounding. However, such related work in robotics do not study reasoning problems, like consistency checking or inference of (missing) relations (e.g., compositions or inverses), in the spirit of the well-studied qualitative spatial reasoning calculi. The lack of formal studies on such reasoning problems might lead to incorrect conclusions. For instance, based on \citeN{zampogiannis2015learning}'s directional relations, \citeN{mota2018incrementally} further define \ii{above} as an inverse of \ii{below} by an ASP rule and rely on it for further inferences. However, according to the definitions of directional relations in these studies, it is not always correct that, for every two objects $A$ and $B$, $A$ is \ii{below} $B$ iff $B$ is \ii{above} $A$ (see~\ref{sec:inverse} for a counter example).
On the other hand, \3ncdc (1) stems from a qualitative spatial calculus of 3D CDC, where computational aspects are well-studied, (2) extends 3D CDC further to 3D-nCDC with nonmonotonic constructs and considering other automated reasoning problems (like inferring missing relations and explanation generation), (3) is sound and complete (Corollary~\ref{cor:correct-basic-3d}), and (4) provides a computational tool to automate reasoning about 3D cardinal directions. In that sense, \3ncdc provides a provably correct method and tool that robotics studies can benefit from.

We have summarized the similarities and differences of our contributions above in comparison with the closely related work in qualitative spatial reasoning about 3D cardinal directional relations (i)--(v), and in applications of ASP to qualitative spatial reasoning, including our earlier studies (a)--(d). We have also discussed related studies about qualitative spatial relations in robotics, and the further needs in robotics for  qualitative spatial reasoning by emphasizing the significance of our contributions (1)--(4).

Further differences from the related work and our earlier work will be pointed out as we provide details about \3ncdc.

\section{3D-nCDC: Nonmonotonic Cardinal Direction Calculus in 3-Dimensional Space}\label{sec:3d-ncdc}

Cardinal Directional Calculus (CDC)~\cite{Goyal1997,SkiadKoub2004,Liuetal2010} describes qualitative direction of an extended spatial object $a$ (the primary or target object) with respect to another object $b$ (the reference object) on a plane, in terms of cardinal directions as follows.
The minimum bounding rectangle of a region $b$, denoted $mbr(b)$, is the smallest rectangle that contains $b$ and has sides parallel to the x and y axes. The minimum bounding rectangle of the reference object $b$ divides the plane into nine regions (called tiles) and these tiles define the nine cardinal directions relative to $b$: \textit{north} (N), \textit{south} (S), \textit{east} (E), \textit{west} (W), \textit{northeast} (NE), \textit{northwest} (NW), \textit{southeast} (SE), \textit{southwest} (SW), \textit{on} (O), as illustrated in Fig.~\ref{fig:relations-2d-3d}(i).
After identifying the unique tiles ${R_1}(b),...,{R_k}(b)$ ($1\leq k\leq 9)$ occupied by the primary object $a$, the direction of $a$ with respect to $b$ is expressed by the basic CDC relation $R_1{:} R_2{:} ...{:} R_k$.

\begin{figure}[t]   
    \begin{tabular}{ccc}
    \includegraphics[width=0.25\textwidth]{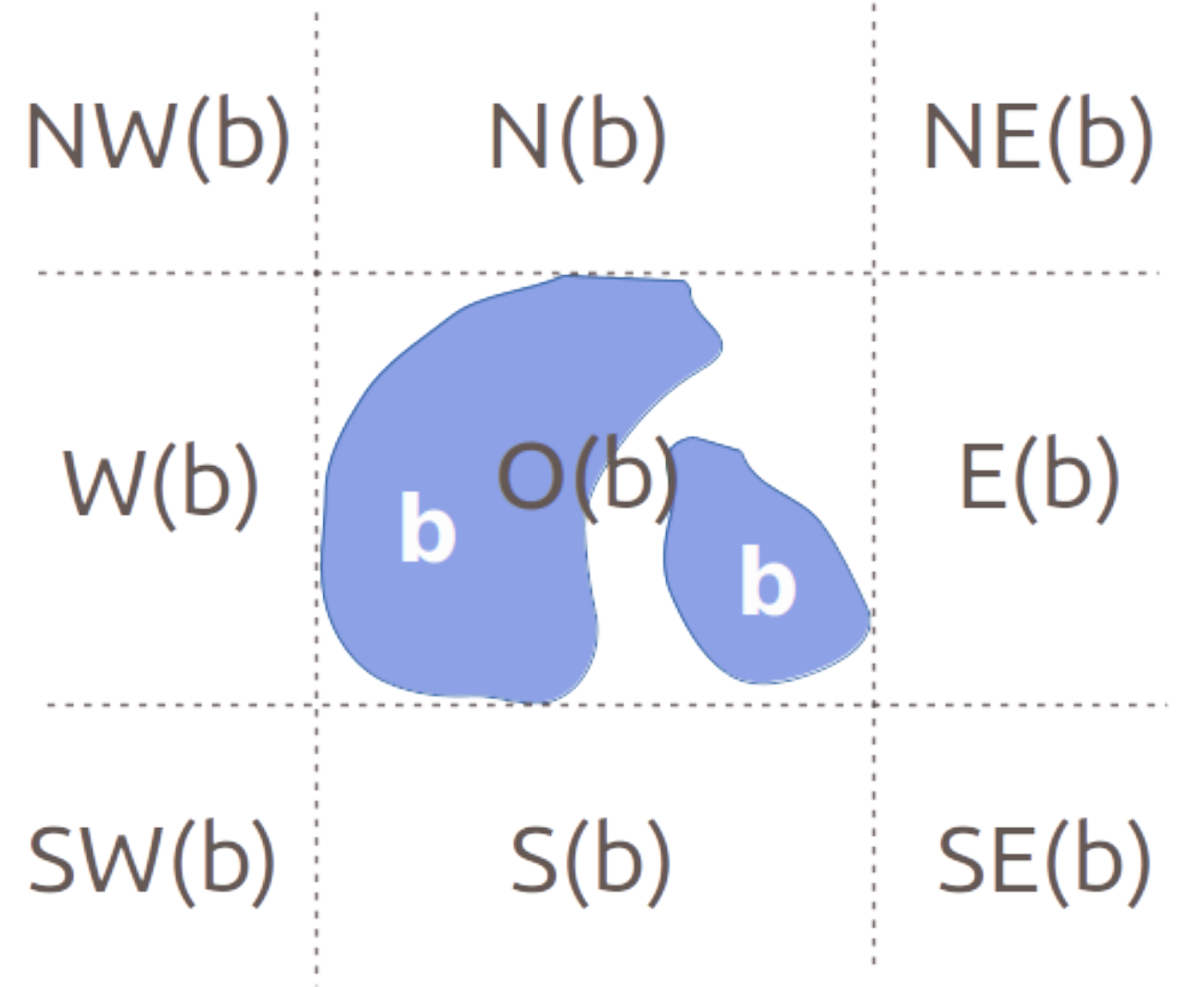} &
    \includegraphics[width=0.3\textwidth]{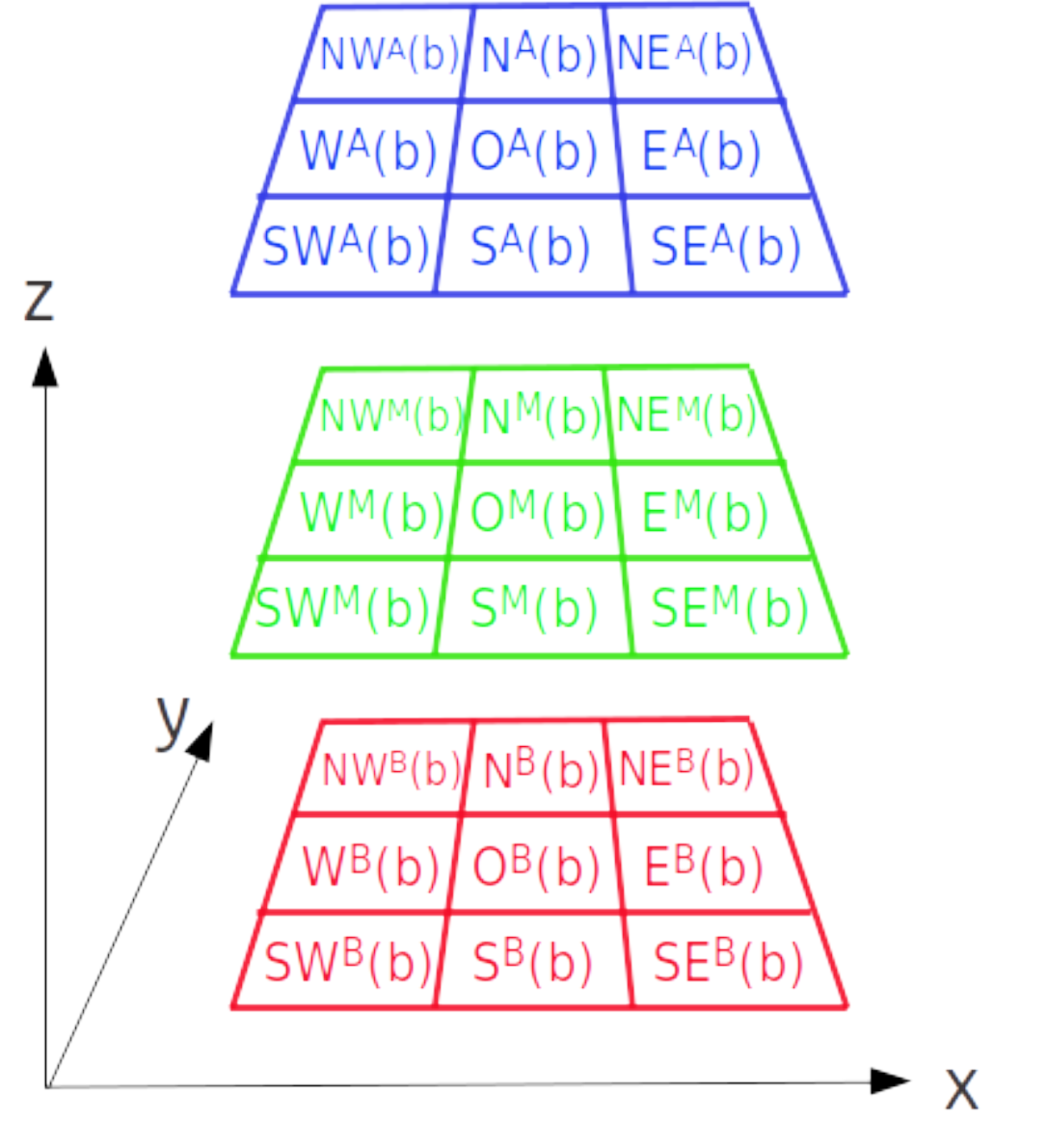} &
    \begin{minipage}[b]{0.3\textwidth}
      \centering\includegraphics[width=0.35\textwidth]{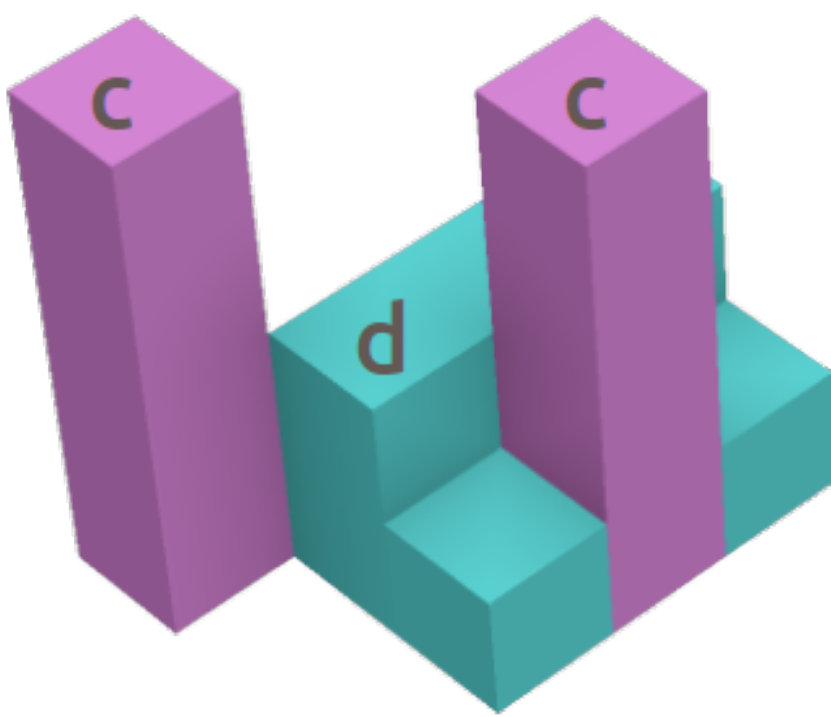} \\
      \centering(iii) \\
      \vspace{0.2cm}
      \centering\includegraphics[width=0.35\textwidth]{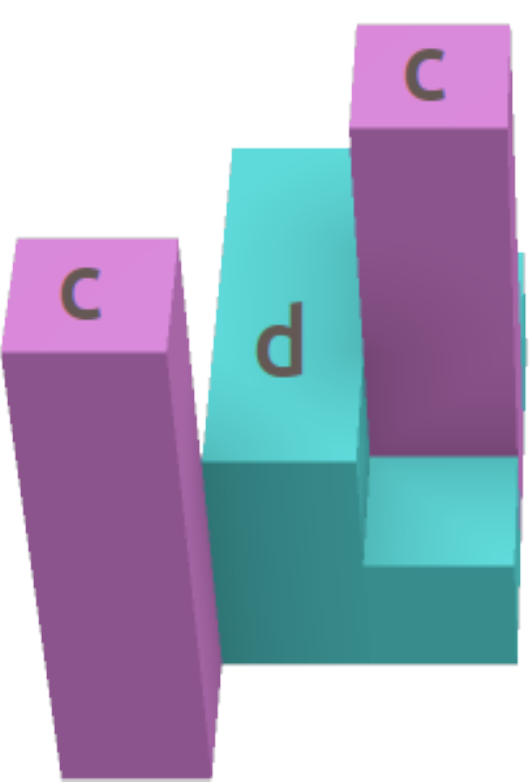}
    \end{minipage}\\
    (i) & (ii) & (iv)
    \end{tabular}
    \vspace{-\baselineskip}
    \caption{\scriptsize{(i) The minimum bounding rectangle of a region $b$, and the 9 single-tiles on the plane relative to $b$. (ii) The 27 single-tiles in 3D relative to object $b$. (iii) Two spatial objects $c$ and $d$. (iv) The spatial objects of (iii) are axes-aligned. The direction of $c$ with respect to $d$ in 3D is represented by the multi-tile 3D nCDC relation $O^{M}:O^{A}:SW^{M}:SW^{A}$.}}
    \label{fig:relations-2d-3d}
    \vspace{-\baselineskip}
\end{figure}

\noindent{\underline{Spatial objects and relations in 3D-nCDC}}
Our study relies on two extensions of CDC to 3D space: Three-dimensional Cardinal Direction (TCD) calculus~\cite{chen2007cardinal}, and Block Cardinal Direction (BCD) calculus~\cite{hou2016reasoning}.
TCD and BCD consider spatial objects that are blocks in 3D space. Different from TCD and BCD, we consider {\em spatial objects} as nonempty, regular, compact volumes in $\mathbb{R}^3$. Spatial objects have positive volume, so lower dimensional entities such as points, lines, surfaces are not considered in 3D nCDC.
A subset of $\mathbb{R}^3$ is {\em regular} if it is equal to closure of its interior; regular objects do not have isolated singular points or emanating lines or planes.
A set is connected if it cannot be stated as union of two disjoint nonempty closed sets.
An object is {\em connected} if its interior is a connected set (so trivial cases where an object has two separate components which touch on a mere single point or line are excluded); connected objects might have holes inside.
An object that is not connected is called {\em disconnected}.
A {\em possibly disconnected object} is a union of finite number of connected objects.
Let $\Reg$ and $\Regs$ denote the set of connected and possibly disconnected objects in $\mathbb{R}^3$, respectively.

Since we consider spatial objects of arbitrary shapes, we describe the direction of a target object $a$ with respect to a reference object $b$, by identifying the minimum bounding box of $b$.
Let $\inf_x(b)$ and $\sup_x(b)$ denote the infimum and supremum of the projection of object $b$ on the x-axis. Similarly, the projections of $b$ on the y and z axes are described by $\inf_y(b)$, $\sup_y(b)$, $\inf_z(b)$, $\sup_z(b)$.
We define the {\em minimum bounding box (mbb)} of an object $b$ as a prism whose sides are described by six planes: $x = \inf_x(b)$, $x = \sup_x(b)$, $y = \inf_y(b)$, $y = \sup_y(b)$, $z = \inf_z(b)$, and $z = \sup_z(b)$. Therefore, the $mbb(b)$ of an object $b$ divides the space into 27 tiles: $NW^{A}(b),...,SE^{A}(b), NW^{M}(b),...,SE^{M}(b), NW^{B}(b),...,SE^{B}(b)$ as illustrated in Fig.~\ref{fig:relations-2d-3d}(ii). Here, the superscripts  $A$, $M$ and $B$ denote three levels on the z-axis: \textit{above}, \textit{middle}, \textit{below}. For example, $N^{B}(b)$ is the tile below and to the north of $b$, and consists of the coordinates $(x,y,z)\in \mathbb{R}^3$ where $\inf_x(b) < x < \sup_x(b)$, $y > \sup_y(b)$, $z < \inf_z(b)$. Note that the tiles are open sets and do not include their boundary points. In TCD and BCD, the objects are already blocks, so $mbb(b)=b$.

As in TCD and BCD, a {\em basic 3D-nCDC relation} $a\ R_1{:}R_2{:}...{:}R_k\ b$ holds if and only if $a \cap {R_i}(b) \neq \emptyset$ for every $1\leq i\leq k$.
For example,  in Fig.~\ref{fig:relations-2d-3d}(iii) (that is axes-aligned in (iv)), $c\ O^{M}:O^{A}:SW^{M}:SW^{A}\ d$. If $k=1$, this basic CDC relation is called a {\em single-tile relation}; if $k\geq 2$, it is called a {\em multi-tile relation}. Let us denote by $\mathcal{R}^s$ the set of single-tile relations, and by $\mathcal{R}$ the set of basic 3D-nCDC relations over $\Regs$.

As in BCD, a {\em disjunctive 3D-nCDC relation} is a finite set $\delta \eqs \{\delta_1,...,\delta_o\}$, ($o>1$) of basic 3D-nCDC relations, intuitively describing their exclusive disjunction. TCD does not consider disjunctive relations.
A {\em 3D-nCDC relation} can be basic or disjunctive.

\noindent{\underline{Basic/disjunctive 3D-nCDC constraints}}
A formula of the form $u\ \delta\ v$, where $u$ and $v$ are spatial variables and $\delta$ is a 3D-nCDC relation, is called a {\em 3D-nCDC constraint}.

A {\em 3D-nCDC constraint network} $C$ is a set of 3D-nCDC constraints $v_i\ \delta\ v_j$, ($v_i\neq v_j$) defined by  a set $V$ of spatial variables $(v_1,...,v_l)$ where variables range over a domain $D$ of spatial objects in $\mathbb{R}^3$, and a set $Q$ of 3D-nCDC relations $\delta$, such that, for every pair $(u_i,u_j)$ of variables in $V$, at most one 3D-nCDC constraint is included in $C$.

A {\em basic 3D-nCDC (constraint) network} consists of solely basic 3D-nCDC constraints.
A basic 3D-nCDC network $C$ is {\em complete} if it includes a unique 3D-nCDC constraint for every pair $(v_i,v_j)$, $i\neq j$ of variables in $V$; otherwise, $C$ is {\em incomplete}.

\noindent{\underline{Consistency checking}}
A pair $(a,b)$ of spatial objects {\em satisfies} a basic 3D-nCDC constraint $u\ \delta\ v$ if $a\ \delta\ b$ holds.
A pair $(a,b)$ of spatial objects {\em satisfies} a disjunctive 3D-nCDC constraint $u\ \delta\ v$ where $\delta=\{\delta_1,...,\delta_o\}$, if $a\ \delta_i\ b$ holds for exactly one $\delta_i \in \delta$.

Let $C$ be a 3D-nCDC network that consists of basic or disjunctive 3D-nCDC constraints specified by variables in $V\eqs\{v_1,...,v_l\}$.
A {\em solution} for $C$ is a set of l-tuples $(a_1, a_2, ..., a_l)$ of spatial objects in $D$ such that every constraint $v_i\ \delta\ v_j$ in $C$ is satisfied by the corresponding pair $(a_i,a_j)$ of spatial objects. If $C$ has a solution then it is called {\em consistent}.

The {\em consistency checking problem} $I\eqs(C,V,D,Q)$ in 3D-nCDC, decides the consistency of $C$.

\begin{theorem}\label{th:complexity-3d}
If $C$ is an incomplete basic 3D-nCDC network, or $C$ is a 3D-nCDC network that includes disjunctive 3D-nCDC constraints over $D \eqs \Regs$, then $I\eqs(C,V,D,Q)$ is an NP-complete problem.
\end{theorem}

\noindent{\underline{Default constraints of 3D-nCDC}} To enable defaults for commonsense reasoning, we introduce {\em default 3D-nCDC constraints}, which are expressions of the form
$$
\ii{default}\ u\ \delta\ v
$$
\noindent where $u$ and $v$ are variables in $V$ and $\delta$ is a basic 3D-nCDC relation in $Q$.
The meaning of default 3D-nCDC constraints is provided in ASP over a discretized space.

\section{Discretized Consistency Checking in 3D-nCDC}
\label{sec:discretized-consistency-3d}

Let $\Lambda_{m,n,p}$ denote the set of unit cubes (called cells) in a prism of size $m{\times} n {\times} p$, aligned with x, y, z axes.
Every cell is identified by its x, y, z coordinates, relative to the origin $(1,1,1)$.
Every spatial object $a$ is described by a nonempty subset $\Lambda_{m,n,p}(a)$ of cells in $\Lambda_{m,n,p}$ occupied by~$a$.

A cell $(x_1,y_1,z_1)$ is a {\em neighbor} of another cell $(x_2,y_2,z_2)$ if  $|x_{1}{-}x_{2}| {+} |y_{1}{-}y_{2}| {+} |z_{1}{-}z_{2}| {=} 1$.
A cell $(x_1,y_1,z_1)$ is connected to another cell $(x_2,y_2,z_2)$ if $(x_1,y_1,z_1)$ is a neighbor of $(x_2,y_2,z_2)$ or $(x_1,y_1,z_1)$ is connected to a neighbor $(x_3,y_3,z_3)$ of $(x_2,y_2,z_2)$.
A spatial object $a$ is {\em connected} in the grid if there exists a (stem) cell in $\Lambda_{m,n,p}(a)$ that is connected to every other cell in $a$.

The projection of an object $b$ on the x-axis is defined by x-coordinates of all cells of $b$ in $\Lambda_{m,n,p}$. Let $\inf^{m,n,p}_x(b)$ and $\sup^{m,n,p}_x(b)$ denote the infimum and supremum of the projection of $b$ on the x-axis. Similarly, the projections of $b$ on the y and z axes are denoted by $\inf^{m,n,p}_y(b)$, $\sup^{m,n,p}_y(b)$, $\inf^{m,n,p}_z(b)$, $\sup^{m,n,p}_z(b)$.
The minimum bounding box $mbb^{m,n,p}(b)$ of a spatial object $b$ in $\Lambda_{m,n,p}$ is the smallest prism in $\Lambda_{m,n,p}$ that contains $b$, that has the following sides parallel to the x, y or z axes: $\inf^{m,n,p}_x(b)$, $\sup^{m,n,p}_x(b)$, $\inf^{m,n,p}_y(b)$, $\sup^{m,n,p}_y(b)$, $\inf^{m,n,p}_z(b)$, $\sup^{m,n,p}_z(b)$.

The prism is partitioned into a set ${R}_{m,n,p}(b)$ of 27 tiles  with respect to minimum bounding box of a reference object $b$.
For example, $N^{B}_{m,n,p}(b)$ is the tile below and to the north of $b$, and consists of the cells $(x,y,z)\in \Lambda_{m,n,p}$ where $\inf^{m,n,p}_x(b) \leqs x \leqs \sup^{m,n,p}_x(b)$, $y {>}\sup^{m,n,p}_y(b)$, $z {<} \inf^{m,n,p}_z(b)$.

Let $D_{m,n,p}$ denote the set of all spatial objects in $\Lambda_{m,n,p}$.
A pair $(a,b)$ of spatial objects in $D_{m,n,p}$ {\em satisfies} a basic 3D-nCDC constraint $u\ \delta\ v$ if
\begin{itemize}
\item[(C1)]
$\ a \cap {R}_{m,n,p}(b) \neq \emptyset$ for every single-tile relation $R$ in $\delta$, and
\item[(C2)]
$\ a \cap {R}_{m,n,p}(b) \eqs \emptyset$ for every single-tile relation $R$ that is not included in $\delta$.
\end{itemize}

A pair $(a,b)$ of spatial objects in $D_{m,n,p}$ {\em satisfies} a disjunctive 3D-nCDC constraint  $u\ \delta\ v$ where $\delta=\{\delta_1,...,\delta_o\}$, if $a\ \delta_i\ b$ holds for exactly one $\delta_i \in \delta$.

Let $C$ be a 3D-nCDC network that consists of basic or disjunctive 3D-nCDC constraints specified by variables in $V\eqs\{v_1,...,v_l\}$.
A {\em solution} for $C$ is a set of l-tuples $(a_1, a_2, ..., a_l)$ of spatial objects in $D_{m,n,p}$ such that every constraint $v_i\ \delta\ v_j$ in $C$ is satisfied by the corresponding pair $(a_i,a_j)$ of spatial objects. If $C$ has a solution then it is called {\em consistent}.

The {\em discretized consistency checking problem} $I_{m,n,p}\eqs(C,V,D_{m,n,p},Q)$ in 3D-nCDC, decides the consistency of $C$.
The following theorem allows us to solve $I$ by declaratively solving $I_{m,n,p}$.

\begin{theorem}\label{th:digital-equiv-3d}
The consistency checking problem $I\eqs(C,V,D,Q)$ over $D \eqs \Regs$ and the discretized consistency checking problem
$I_{m,n,p}\eqs(C,V,D_{m,n,p},Q)$ where $m,n,p \geqs 2|V|-1$ have the same answers.
\end{theorem}

It is important to emphasize here that we discretize the consistency checking problem, not the environment. For example, given a consistency checking problem with a set of qualitative spatial constraints about a building design (as mentioned in the introduction), we do not discretize the building itself; rather we try to solve the discretized consistency checking problem over a 3D grid of appropriate size.
We do not process grounded numerical spatial data or instantiate cardinal directions over real numbers either.

\section{Discretized Consistency Checking in 3D-nCDC using ASP}
\label{sec:3d-ncdc-asp}

Let $I_{m,n,p} \eqs (C,V,D_{m,n,p},Q)$ be a discretized 3D-nCDC consistency checking problem, where $C$ consists of 3D-nCDC constraints and might be incomplete, and $D_{m,n,p}$ is the set of all spatial objects in $\Lambda_{m,n,p}$ that may be disconnected and have holes. In the following, we incrementally describe an ASP program to solve $I_{m,n,p}$. A brief review of ASP is provided in Appendix.

\vspace{-2ex}
\subsection{Basic 3D-nCDC Networks} \label{sec:basic-3d}

Suppose that $C$ contains basic 3D-nCDC constraints only.  Let us describe the ASP program $\Pi_{m,n,p}^{1}$ that solves $I_{m,n,p}$.

1) We describe every basic 3D-nCDC constraint $u\ \delta\ v$ in $C$, by atoms of the form $\ii{rel}(u,v,r)$ for each single-tile relation $r$ in $\delta$.
Then, $C$ can be represented by a set $F_B$ of facts:
\beq
\ii{rel}(u,v,r) \lar \quad (r\in \delta,\ u\ \delta\ v \in C).
\eeq{eq:rel-3d}
For example, a basic 3D-nCDC constraint $a \ N^A:NW^M\ b$ is represented in ASP by the facts:
$$\ii{rel}(a,b,\ii{NA}).\quad \ii{rel}(a,b,\ii{NWM}).$$

2) A $mbb^{m,n,p}(u)$ is generated for every spatial object $u$, by nondeterministically identifying the infimum/supremum of its projection on the x axis with the choice rules:
\beq
\ba l
\{\ii{inf}_{x}(u, \underline{x}) :\ 1\leqs \underline{x} \leqs m\} \eqs 1 \:  \lar  \quad (u\in V) \\
\{\ii{sup}_{x}(u, \overline{x}) :\ 1\leqs \overline{x} \leqs m\} \eqs 1 \:  \lar  \quad (u\in V).
\ea
\eeq{eq:generate-infsup-3d}

\noindent ensuring that the infimum is less than or equal to the supremum:
\beq
\ba l
\lar \ii{inf}_{x}(u,\underline{x}), \, \ii{sup}_{x}(u,\overline{x})   \quad (\underline{x} \gts \overline{x}, \: u \ins  V).
\ea
\eeq{eq:infsup-ineq-3d}
Similar rules are added for the infimum/supremum of its projection on y and z axes.

3) We instantiate every variable $u\in V$ by a spatial object in $D_{m,n,p}$, by nondeterministically assigning some cells $(x,y,z)$ in $\Lambda_{m,n,p}$ to $u$ so that (i) the minimum bounding box of this object is exactly $mbb^{m,n,p}(u)$ generated by rules~$(\ref{eq:generate-infsup-3d})\cups (\ref{eq:infsup-ineq-3d})$, and (ii) the 3D-nCDC constraints in $C$ are satisfied.

3)(i) An assignment of cells $(x,y,z)$ to a variable $u$ is described by atoms of the form~$\ii{occ}(u,x,y,z)$, nondeterministically generated by the choice rules:
\beq
\{\ii{occ}(u,x,y,z):\ (x,y,z) \in \Lambda_{m,n,p}\} \geqs 1 \: \lar   \quad (u\in V).
\eeq{eq:generate-3d}

Projection of this spatial object onto x axis are defined by the rules:
\beq
\ba l
\ii{xocc}(u,x) \lar \ii{occ}(u,x,y,z)   \quad ((x,y,z)\in \Lambda_{m,n,p}, \: u\in V).
\ea
\eeq{eq:project-mp-3d}
Similar rules are added for its projection on the y and z axes.

We ensure that, for x axis, the projected coordinates lie between the infimum and supremum,
\beq
\ba l
\lar \ii{inf}_{x}(u,\underline{x}), \, \ii{xocc}(u,x')  \quad (x' \lts \underline{x}, \: 1\leqs x' \leqs m, \: u\ins V)  \\
\lar \ii{sup}_{x}(u,\overline{x}), \, \ii{xocc}(u,x')  \quad (x' \gts \overline{x}, \: 1\leqs x' \leqs m, \: u\ins V)
\ea
\eeq{eq:cells-inside-mbr-3d}
\noindent at least one of the cells assigned to $u$ is on the infimum, and another one on the supremum.
\beq
\ba l
\lar \ii{not} \: \ii{xocc}(u, \underline{x}),  \ii{inf}_{x}(u,\underline{x})   \quad (u\ins V)  \\
\lar \ii{not} \: \ii{xocc}(u, \overline{x}),  \ii{sup}_{x}(u,\overline{x})   \quad (u\ins V).
\ea
\eeq{eq:infsup-existcell-3d}
Similar constraints are added for its projection on the y and z axes.

3(ii) We ensure that the instantiations of objects (by assignment of cells $(x,y,z)$ to variables $u\in V$) satisfies every basic 3D-nCDC constraint $u\ \delta\ v$ in $C$.  For that, we add constraints to ensure that conditions (C1) and (C2) are not violated.

For example, if $\delta$ contains the single tile relation $N^{B}$ then we add the following to describe when condition (C1) for $N^{B}$ is violated (i.e., when $u$ does not occupy any cells to the north of and below $mbb^{m,n,p}(v)$).
\beq
\ba l
\ii{violated}(u,v) \lar \ii{rel}(u,v,NB),\ \ii{inf}_{x}(v,\underline{x}),\ \ii{sup}_{x}(v,\overline{x}),\ \ii{sup}_{y}(v,\overline{y}),\ \ii{inf}_{z}(v,\underline{z}), \\
\quad \#\ii{count}\{x,y,z{:}\ \ii{occ}(u,x,y,z), \underline{x} \leqs x \leqs \overline{x},\ y{>} \overline{y},\ z{<} \underline{z},\ (x,y,z) \ins \Lambda_{m,n,p}\} \leqs 0 \quad (u\in V).
\ea
\eeq{eq:c1-3d}
If $\delta$ does not contain $N^{B}$, then the following rules describe when condition (C2) is violated (i.e., $u$ occupies some cells to the north of and below $mbb^{m,n,p}(v)$).
\beq
\ba l
\ii{violated}(u,v) \lar  \#\ii{count}\{x,y,z{:}\ \ii{occ}(u,x,y,z), \underline{x} \leqs x \leqs \overline{x},\ y{>} \overline{y},\ z{<} \underline{z},\ (x,y,z) \ins \Lambda_{m,n,p}\} \geqs 1,\\
\quad \no\ \ii{rel}(u,v,NB),\ \ii{existrel}(u,v),\ \ii{inf}_{x}(v,\underline{x}),\ \ii{sup}_{x}(v,\overline{x}),\ \ii{sup}_{y}(v,\overline{y}),\ \ii{inf}_{z}(v,\underline{z}) \quad (u\in V).
\ea
\eeq{eq:c2-3d}

\noindent Here, since the network $C$ might be incomplete, $\ii{existrel}(u,v)$ atoms identify which pair of variables have a constraint in the network $C$:
\beq
\ii{existrel}(u,v) \lar \ii{rel}(u,v,r) \quad (r\in \mathcal{R}^s,\ u, v \in V).
\eeq{eq:existrel-3d}
For every one of 26 other single tile relations, we add rules similar to (\ref{eq:c1-3d}) and (\ref{eq:c2-3d}).
After that, we eliminate such violations:
\beq
\lar \ii{violated}(u,v), \ii{existrel}(u,v) \quad (u, v \in V).
\eeq{eq:violated}

The ASP program $\Pi_{m,n,p}^{1}$ described above (including the ASP description $F_B$ of $C$) for checking the consistency of a basic 3D-nCDC network $C$ over $D_{m,n,p}$ is sound and complete. Let $\mathcal{O}_{m,n,p}$ denote the set of atoms of the form $\ii{occ}(u,x,y,z)$ where $u\in V$ and $x$, $y$, $z$ are positive integers such that $1\leq x\leq m$, $1\leq y\leq n$, $1\leq z\leq p$.

\begin{theorem}\label{th:thm-correct-basic-3d}
	Let $I_{m,n,p} \eqs (C,V,D_{m,n,p},Q)$ be a discretized consistency checking problem, where $C$ is a basic 3D-nCDC network.
	For an assignment $X$ of spatial objects in $D_{m,n,p}$ to variables in $V$, $X$ is a solution of $I_{m,n,p}$ if and only if $X$ can be represented in the form of $X \eqs Z \cap \mathcal{O}_{m,n,p}$ for some answer set $Z$ of $\Pi_{m,n,p}^{1}$.
	Moreover, every solution of $I_{m,n,p}$ can be represented in this form in only one way.
\end{theorem}

\noindent From Theorems~\ref{th:digital-equiv-3d} and~\ref{th:thm-correct-basic-3d}:

\begin{corollary}\label{cor:correct-basic-3d}
    The consistency checking problem $I\eqs(C,V,D,Q)$ has a solution if and only if the program $\Pi_{m,n,p}^{1}$ ($m,n,p \geqs 2|V|-1$) has an answer set.
\end{corollary}

\vspace{-3ex}
\subsection{Disjunctive 3D-nCDC Constraints} \label{sec:disjunctive-3d}

Suppose that $C$ contains basic or disjunctive 3D-nCDC constraints only.  Let us describe the ASP program $\Pi_{m,n,p}^{2}$ that solves $I_{m,n,p}$. The program $\Pi_{m,n,p}^{2}$ is obtained from $\Pi_{m,n,p}^{1}$, by adding new rules for each disjunctive 3D-nCDC constraint as follows.

1) Every disjunctive 3D-nCDC constraint $u\ \{\delta_1, ..., \delta_o\}\ v$ in $C$ is represented in ASP by a set $F_V$ of facts:
\beq
\ba l
\ii{disjrel}(u,v,i,r) \lar \quad (r \in \delta_i, \: 1\leq i \leq o).
\ea
\eeq{eq:cdc-disj-rel-3d}

2) Recall that a pair $(a,b)$ of spatial objects satisfies $u\ \delta\ v$ where $\delta=\{\delta_1,...,\delta_o\}$, if $a\ \delta_i\ b$ holds for exactly one $\delta_i \in \delta$.  Therefore, for every disjunctive 3D-nCDC constraint $u\ \delta\ v$, we nondeterministically choose $\delta_i \in \delta$, and represent the basic 3D-nCDC constraint $u\ \delta_i\ v$:
\begin{align}
\{\ii{chosen}(u,v,i): 1\leq i \leq o \} \eqs 1 \: \lar  \label{eq:disj-choose-3d} \\
\ii{rel}(u,v,R) \lar \ii{chosen}(u,v,i), \: \ii{disjrel}(u,v,i,R).  \label{eq:rel-disj-3d}
\end{align}

The ASP program $\Pi_{m,n,p}^{2}$ is sound and complete.
\begin{theorem}\label{th:asp-correct-disj-3d}
	Let $I_{m,n,p} \eqs (C,V,D_{m,n,p},Q)$ be a discretized consistency checking problem, where $C$ contains basic or disjunctive 3D-nCDC constraints.
	For an assignment $X$ of spatial objects in $D_{m,n,p}$ to variables in $V$, $X$ is a solution of $I_{m,n,p}$ if and only if $X$ can be represented in the form of $X \eqs Z \cap \mathcal{O}_{m,n,p}$ for some answer set $Z$ of $\Pi_{m,n,p}^{2}$.
	Moreover, every solution of $I_{m,n,p}$ can be represented in this form in only one way.
\end{theorem}

\vspace{-3ex}
\subsection{Default 3D-nCDC Constraints}    
\label{sec:default_constr_3d}

Suppose that $C$ also contains default 3D-nCDC constraints.  Let us describe the ASP program $\Pi_{m,n,p}^{3}$ that solves $I_{m,n,p}$. The program $\Pi_{m,n,p}^{3}$ is obtained from $\Pi_{m,n,p}^{2}$, by adding new rules for each default 3D-nCDC constraint as follows.

1) We represent every default 3D-nCDC constraint $\ii{default}\ u\ \delta \ v$ (where $\delta$ is a basic relation) by a set $F_D$ of facts:
\beq
\ba l
\ii{defaultrel}(u,v,r) \lar \quad (r \in \delta).
\ea
\eeq{eq:cdcp-default-rel-3d}

2) The default 3D-nCDC constraint $\ii{default}\ u\ \delta \ v$ applies if there is no evidence against it:
\beq
\ii{drel}(u,v) \lar \no\ \neg \ii{drel}(u,v), \ii{defaultrel}(u,v,r) \quad (r\in \delta).
\eeq{eq:def-apply-3d}

3) The evidence against a default constraint $\ii{default}\ u\ \delta \ v$ can be due to violations of conditions (C1) and (C2), which are defined by atoms of the form $\ii{violatedDef}(u,v)$ similar to atoms $\ii{violated}(u,v)$: use \ii{defaultrel} instead of \ii{rel}. For example, if $\delta$ contains the single-tile relation $N^{B}$ then we add the following rules to describe when condition (C1) for $N^{B}$ is violated.
\beq
\ba l
\ii{violatedDef}(u,v) \lar \ii{defaultrel}(u,v,NB),\ \ii{inf}_{x}(v,\underline{x}),\ \ii{sup}_{x}(v,\overline{x}),\ \ii{sup}_{y}(v,\overline{y}),\ \ii{inf}_{z}(v,\underline{z}), \\
\quad \#\ii{count}\{x,y,z{:}\ \ii{occ}(u,x,y,z), \underline{x} \leqs x \leqs \overline{x},\  y{>} \overline{y},\ z{<} \underline{z},\ (x,y,z) \ins \Lambda_{m,n,p}\} \leqs 0  \quad (u\in V).
\ea
\eeq{eq:c1-3d-def}
If $\delta$ does not contain $N^{B}$, then the following rules describe when condition (C2) is violated.
\beq
\ba l
\ii{violatedDef}(u,v) \lar  \no\ \ii{defaultrel}(u,v,NB),\ \ii{existDefRel}(u,v),\\
\quad \#\ii{count}\{x,y,z{:}\ \ii{occ}(u,x,y,z), \underline{x} \leqs x \leqs \overline{x},\ y{>} \overline{y},\ z{<} \underline{z},\ (x,y,z) \ins \Lambda_{m,n,p}\} \geqs 1, \\
\quad \ii{inf}_{x}(v,\underline{x}),\ \ii{sup}_{x}(v,\overline{x}),\ \ii{sup}_{y}(v,\overline{y}),\ \ii{inf}_{z}(v,\underline{z}) \quad (u\in V).
\ea
\eeq{eq:c2-3d-def}

For every one of 26 other single tile relations, we add rules similar to (\ref{eq:c1-3d-def}) and (\ref{eq:c2-3d-def}).

4) Then, the evidence against a default 3D-nCDC constraint $\ii{default}\ u\ \delta \ v$  via such violations can be defined as follows:
\beq
\ba l
\neg \ii{drel}(u,v) \lar \ii{violatedDef}(u,v), \ii{existDefRel}(u,v) \\
\xleftarrow{\scriptstyle\sim} \: \neg \ii{drel}(u,v), \ii{existDefRel}(u,v) \quad [1@1,u,v]
\ea
\eeq{eq:violated-def}
where $\ii{existDefRel}(u,v)$ is defined as follows:
\beq
\ii{existDefRel}(u,v) \lar \ii{defaultrel}(u,v,r) \quad (r\in \mathcal{R}^s,\ u, v \in V).
\eeq{eq:existdefrel-3d}

The weak constraint above minimizes the evidences provided by abductive inferences of occupied cells.
The rule aims to satisfy as many default 3D-nCDC constraints as possible, so as not to conflict with the other 3D-nCDC constraints in $C$.

5) The evidence (or abnormal cases) against a default 3D-nCDC constraint can be provided by the user. Consider, for instance, a building whose entrance is from its ceiling; then, the abnormal entrance provides an exception to a default constraint that expresses that the ``normally, the terrace is above the entrance''. This exception can be expressed as follows:
$$
\ba l
\neg \ii{drel}(u,v) \lar \ii{ab}(v), \ii{existDefRel}(u,v)\\
\neg \ii{drel}(u,v) \lar \ii{ab}(u), \ii{existDefRel}(u,v)\\
\ii{ab}(\ii{Entrance}) \lar .
\ea
$$

For every answer set $Z$ for $\Pi_{m,n,p}^{3}$, the assumption expressed by a default 3D-nCDC constraint $\ii{default}\ u\ \delta \ v$ {\em applies} if there is no exception $\ii{drel}(u,v)$ in $Z$ against the default.

\section{Connected Spatial Objects} \label{sec:connected-3d}

Until now, we have assumed that objects belong to $\Regs$, and they can be disconnected. In many real-world applications, spatial objects are connected (and thus belong to $\Reg$). We ensure connectedness of these objects, by adding the following rules to $\Pi_{m,n,p}^{3}$.

For each spatial object, we formulate the concept of connectedness by incrementally defining its connected cells starting from one cell (called the stem cell), and then enforce all the cells of the object to be reachable from this stem cell. Note that it is sufficient to check the connectedness only for objects which act as target variables in some constraint in $C$. The connectedness of other objects can be accomplished by freely constructing them inside their minimum bounding boxes.

1) Let $Trg_C\subseteq V$ be the set of variables that appear as a target object in some constraint in $C$.
We define the stem cell for each target spatial object $u\in Trg_C$, as the left bottom below corner cell of the object. First, we find the cells with minimum x coordinate:
\beq
\ba l
\ii{left-side}(u,y,z) \lar \ii{inf}_{x}(u,\underline{x}), \, \ii{occ}(u,\underline{x},y,z)  \quad (1 \leqs y \leqs n, \: 1 \leqs z \leqs p, \: u\ins Trg_C) \\
\ii{left-border}(u,y) \lar \ii{inf}_{x}(u,\underline{x}), \, \ii{occ}(u,\underline{x},y,z)  \quad (1 \leqs y \leqs n, \: 1 \leqs z \leqs p, \: u\ins Trg_C).
\ea
\eeq{eq:left-bordery-3d}
Then, among these cells, we find the cells with the minimum y coordinate
\beq
\ii{ymin}(u,y_m) \lar \# \ii{min} \: \{y: \; \ii{left-border}(u,y)\, \}  \eqs y_m \quad (u\ins Trg_C).
\eeq{eq:lymin-3d}
Then, among these cells, we pick the cell with minimum z coordinate:
\beq
\ba l
\ii{zborder}(u,z) \lar \ii{left-side}(u,y_m,z), \, \ii{ymin}(u,y_m)  \quad ( u\ins Trg_C)  \\
\ii{zmin}(u,z_m) \lar \# \ii{min} \: \{z: \; \ii{zborder}(u,z)\, \}  \eqs z_m  \quad (u\ins Trg_C).   
\ea
\eeq{eq:zmin-3d}
Then, we define the stem cell as follows:
\beq
\ba l
\ii{stem}(u,\underline{x},y_m,z_m) \lar \ii{inf}_{x}(u,\underline{x}), \, \ii{ymin}(u,y_m), \, \ii{zmin}(u,z_m)  \quad (u\ins Trg_C).  
\ea
\eeq{eq:stem-cell-3d}

2) For every target spatial object $u\in Trg_C$), we define a set of connected cells starting from the stem cell:
\beq
\ba l
\ii{connset}(u,x,y,z) \lar \ii{stem}(u,x,y,z)  \quad (u\ins Trg_C).  \\
\ii{connset}(u,x_2,y_2,z_2) \lar \ii{connset}(u,x_1,y_1,z_1), \ii{occ}(u,x_2,y_2,z_2) \\
\quad (|x_2-x_1|+|y_2-y_1|+|z_2-z_1|=1,\ u\ins Trg_C).
\ea
\eeq{eq:connected-closure-3d}

3) We ensure that every cell of $u$ belongs to the connected set:
\beq
\ba l
\lar not \: connset(u,x,y,z), \, \ii{occ}(u,x,y,z)  \quad (u\ins Trg_C).
\ea
\eeq{eq:conn-test-3d}

\vspace{-2ex}
\section{Inferring Missing 3D-nCDC Relations}\label{sec:infer-3d}

Let $Z$ be an answer set for $\Pi_{m,n,p}^{3}$.  For every pair of different spatial objects $a$ and $b$, we say that $a$ and $b$ are {\em related} by a 3D-nCDC relation in $Z$ if there exists an atom $\ii{rel}(a,b,r)$ for some single-tile relation $r\in \mathcal{R}^s$ in $Z$, or a $\ii{drel}(a,b)$ atom in $Z$.
Otherwise, we say that there is a missing relation between $a$ and $b$. In such cases (e.g., to explain the relative direction between two objects), it is beneficial to infer the missing relations.

1) Suppose that the user specifies which missing relations $(u,v)$ shall be inferred, by a set $F_I$ of facts of the form
$\ii{toinfer}(u,v).$

2) To infer a missing relating between two different spatial objects $u$ and $v$, we nondeterministically generate a basic 3D-nCDC relation $\delta$ that consists of single-tile relations $r$:
\beq
\ba l
\ii{known}(u,v) \lar \ii{existrel}(u,v). \\
\ii{known}(u,v) \lar \ii{drel}(u,v).  \\    
\{\ii{infer}(u,v,r) : r \in \mathcal{R}^s \} \geqs 1 \: \lar  \no\ \ii{known}(u,v), \ii{toinfer}(u,v).
\ea
\eeq{eq:infer-3d}

3) We add rules similar to (\ref{eq:c1-3d}), (\ref{eq:c2-3d}) and (\ref{eq:violated}), using $\ii{infer}$ atoms instead of $\ii{rel}$ atoms, $\ii{inferViolated}$ atoms instead of $\ii{violated}$ atoms, and $\ii{existInfer}$ atoms instead of $\ii{existrel}$ atoms, to ensure the conditions (C1) and (C2) for each inferred single-tile relation.

Let $\Pi_{m,n,p}^{3,f}$ be the program obtained from $\Pi_{m,n,p}^{3}$ as described above (including $F_I$).  The atoms of the form $\ii{infer}(u,v,r)$ in an answer set for $\Pi_{m,n,p}^{3,f}$ describe {\em inferred 3D-nCDC relations}.

\section{Explaining Inconsistencies in 3D-nCDC}\label{sec:source-inconsistency-3d}

If the constraint network $C$ is inconsistent, constraints are not satisfiable all together.
However, when we exclude some constraints, the network may become consistent. In that sense, the set of excluded constraints are a source of inconsistency in the original network $C$.

To find a source of inconsistency, we replace constraints (\ref{eq:violated}) with the weak constraints:
\beq
\xleftarrow{\scriptstyle\sim}  \ii{violated}(u,v), \ii{existrel}(u,v)\ [1@2,u,v] \quad (u, v \in V).
\eeq{eq:violated-weak}

\noindent According to this weak constraint, each violated 3D-nCDC constraint has a cost of 1, and the number of violated constraints are optimized with priority 2.

Let $Z$ be an answer set for the program obtained from $\Pi_{m,n,p}^{3}$ by replacing (\ref{eq:violated}) with (\ref{eq:violated-weak}).
Then, the set $E_Z$ atoms of the form $\ii{violated}(u,v)$ that appear in $Z$ describes the basic/disjunctive constraints $u\ \delta\ v$ in $C$ that are violated; furthermore if these constraints are excluded, then $C$ would be consistent. Therefore, we say that $E_Z$ provides an {\em explanation} for the inconsistency of the network $C$.

Note that the inconsistency might be due to the violation of mandatory constraints or users' requests/preferences.
Since the mandatory constraints cannot be changed, it might be better to explain inconsistencies in terms of the violations of user's requests/preferences, by replacing (\ref{eq:violated}) with the following weak constraints (instead of (\ref{eq:violated-weak})):
\beq
\ba l
\lar \ii{violated}(u,v), \: \ii{mandatory}(u,v), \:  \ii{existrel}(u,v) \quad (u, v \in V)   \\
\xleftarrow{\scriptstyle\sim}  \ii{violated}(u,v), \: \no\ \ii{mandatory}(u,v), \: \ii{existrel}(u,v)\ [1@2,u,v] \quad (u, v \in V).
\ea
\eeq{eq:violated-weak-mandatory}
\noindent Such an explanation is illustrated with an example in Section~\ref{sec:building}.

Note that the weak constraint above allows us to find minimal explanations. The priority of the weak constraints in (\ref{eq:violated-weak}), (\ref{eq:violated-weak-mandatory}) is higher than the priority of the weak constraints utilized by the default constraints~(\ref{eq:violated-def}), since consistency checking is prioritized.

Since the explanations are provided in terms of violations of constraints/preferences specified by the user, they can be presented to the user in an understandable format in the same way as constraints/preferences are specified. For instance, if the 3D-nCDC constraint $\ii{Director}\ O^A\ \ii{Entrance}$ is specified by the user as a request in natural language as follows ``The director's office is placed above the entrance,'' and the answer set $Z$ includes the atom $\ii{violated}(\ii{Director},\ii{Entrance})$, then an explanation for the inconsistency of the network (i.e., that the design of the building with respect to the given constraints and preferences is not possible) can be presented to the user also in natural language matching with his/her own specification: ``... because the director's office cannot be placed above the entrance.''  If the user specifies his/her requests via a graphical user interface, then the requests that cannot be fulfilled could instead be highlighted by red color.

\section{Applications of \3ncdc}
\label{sec:scenarios-3d}

We discuss the usefulness of \3ncdc by three interesting real-world applications: marine explorations with an underwater human-robot team, building design and regulation in architecture, and evidence-based digital forensics.

\vspace{-2ex}
\subsection{Marine Exploration with Underwater Robots}
\label{sec:marine}

The application presented in this section is motivated by the challenges of 3D localization and natural human-robot communication in underwater robotics and marine exploration~\cite{zereik2018}. 
Below a certain depth, GPS does not function and sunlight cannot penetrate, so obtaining exact and absolute locations of objects is not possible. Topographical entities may be discontinuous and precise boundaries are often not clear, so agents need to describe rough positions of the entities in the fauna relative to one another.

Suppose that a group of researchers and underwater robots are in a mission to discover a biological habitat in ocean basin. The environment is unknown to them. During this exploration, Researcher~1 is investigating the sedimentary rock, Robot~1 is checking the fragmented marsh, which is below the sedimentary rock to its southwest and southeast, and Robot~2 is at the thermal zone, which is above the sedimentary rock to its east and southeast. Robot~2 reports the existence of a semi-active volcanic vent, located above the marsh to its northeast. Researcher~2 finds a kelp forest with two separated parts: one part is located to the north and the other part is located to the southeast of the volcanic vent, both parts are located at a lower depth. Robot~3 discovers a fungi culture to the south of the kelp forest on the same level, and to the east and below of the marsh. The fungi culture is of interest to Researcher~1 but which direction should he proceed to reach it?

The qualitative spatial information provided by the four agents can be encoded as a 3D-nCDC constraint network as follows:
$$
\ba{lll}
\ii{Marsh}\ SW^B{:}SE^B\ \ii{SedRock} &
\ii{Volcano}\ E^A{:}SE^A\ \ii{SedRock} &
\ii{Volcano}\ NE^A\ \ii{Marsh}  \\
\ii{Kelp}\ N^B{:}SE^B\ \ii{Volcano} &
\ii{Fungi}\ S^M\ \ii{Kelp} &
\ii{Fungi}\ E^B\ \ii{Marsh}.
\ea
$$

The goal is to infer the relation of the fungi culture with respect to the location of Researcher~1, the sedimentary rock. For that purpose, we consider the program $\Pi_{m,n,p}^{3}$, including a set $F_B$ of facts ~(\ref{eq:rel-3d}) describing the basic 3D-nCDC constraints above, and the fact
$\ii{toinfer}(\ii{Fungi},\ii{SedRock}).$
In every answer set for this program, atoms of the form $\ii{infer}(\ii{Fungi},\ii{SedRock},r)$ reveal a possible location of the fungi culture with respect to the sedimentary rock. For instance, one of these answer sets computed by \clingo includes $\ii{infer}(\ii{Fungi},\ii{SedRock},SEB)$, leading to the inferred 3D-nCDC constraint  $\ii{Fungi}\ SE^B\ \ii{SedRock}$.  Then, Researcher~1 can be guided towards southeast and below, to find the fungi culture.

\vspace{-2ex}
\subsection{Building Design and Regulation}
\label{sec:building}

The application presented in this section is motivated by the challenges of building design and regulations in architecture.
As argued in~\cite{borrmann2010towards}, legal requirements and official regulations together with client demands about housing, rooms and equipment inside the building are usually documented using qualitative words of daily language rather than mathematical formulas. For this reason, qualitative spatial reasoning is required.

Suppose that an architect is designing a multi-floor library building. The entrance corridor and the door are in the ground floor, and south (middle front) of the building. The regulations impose the electric panel to be on the same floor or a lower level than the entrance. The electric panel must also be situated next to the main cable, which is at the north side of the building. The system room can be on another floor, however, for ease of cabling along the shaft, it must be vertically aligned with the electric panel. The heating unit is normally instituted on a lower level, and southwest to the entrance. Moreover, the library director requests her office to be on and above the entrance corridor, for convenience of monitoring. She also requests that the system room be located to the left of her office on the same floor. The presumed location of the secretary is to the right of the director's office.
Is it possible to come up with a design of this library to respect all these constraints, requests, and assumptions?

The spatial requirements of the building design describe above can be specified by the following 3D-nCDC constraint network:
$$
\ba{lll}
\ii{Panel}\ \{ N^M, \, N^B \}\ \ii{Entrance} &
\ii{System}\ \{ O^B, \, O^M, \, O^A \}\ \ii{Panel} &
\ii{Director}\ O^A\ \ii{Entrance} \\
\ii{System}\ W^M\ \ii{Director} &
\ii{default}\ \ii{Heating}\ SW^B\  \ii{Entrance} &
\ii{default}\ \ii{Secretary}\ E^M\  \ii{Director}.
\ea
$$

With the program $\Pi_{m,n,p}^{3}$, including a set $F_B \cup F_V \cup F_D $ of facts describing the 3D-nCDC constraints above, this constraint network is found inconsistent by \clingo.
To explain this inconsistency, we utilize the method explained in Section \ref{sec:source-inconsistency-3d}: replace the constraints (\ref{eq:violated}) in $\Pi_{m,n,p}^{3}$ with the weak constraints (\ref{eq:violated-weak-mandatory}), where
$\ii{mandatory}(Panel,Entrance)$ given in the input represents an official regulation.
An answer set computed for this program by \clingo includes the atom $\ii{violated}(\ii{Director},\ii{Entrance})$, and thus provides the following explanation: the director's request about the location of her office (i.e., the 3D-nCDC constraint $\ii{Director}\ O^A\ \ii{Entrance}$) cannot be fulfilled with respect to the other desired features of the library.

\vspace{-2ex}
\subsection{Evidence-Based Digital Forensics}
\label{sec:df-application-3d}

The application presented in this section is motivated by the challenges of evidence-based digital forensics~\cite{CostantiniGO19}, that goes beyond data analysis.
We consider a fictional crime story inspired by Agatha Christie's novel ``Hercule Poirot's Christmas''. Suppose that the grandfather of the Lee family is murdered.

The police obtains some images of the crime scene from the cameras located in the house. The images yield the following information
at the moment of the crime:
\beq
\ba{lll}
\ii{Body}\ S^M:SE^M\ \ii{Table}    & \ii{Teapoy}\ E^M\ \ii{Sofa}   & \ii{Suitcase}\ \{ S^M, SW^M \}\ \ii{Table} \\
\ii{Body}\ N^M:NE^M\ \ii{Teapoy}   & \ii{Phone}\ O^A \ii{Table}    &  \ii{Sofa}\ SE^M\ \ii{Bed}\ \ \ \ \ii{Coat}\ O^M\ \ii{Hanger} .
\ea
\eeq{eq:digital}
Notice that, since some images are not clear, there is some uncertainty regarding the position of the suitcase.
Meanwhile, the detective Poirot interviews the two suspects of the crime.

Suspect 1 (Pilar):  ``... Suddenly, some noise and a scream came from upstairs. I immediately went to my grandfather's bedroom and found him dead on the floor. His body was lying in front of the table, a bit to the right. There was a rope hanging up on the window that is behind the body, which is strange. There was a muffler on top of the drawer, which probably belongs to my grandfather. The phone book on the table was open. Also, I saw a whistle and toy balloon on the floor, next to the body to its right, that is somehow peculiar... ''

Suspect 2 (Alfred):  ``... I was sitting in the guest room with Stephan. I heard a noise and then ran upstairs to my father's bedroom.
The room was untidy. Probably someone else had visited him before because I noticed a suitcase in front of the table.
I saw some drugs on the teapoy. There was a knife on the floor next to the body, to its right. It was to the front and underneath the phone..."

From Suspect 1's statement, the following 3D-nCDC constraints are obtained:
\beq
\ba{lll}
\ii{Body}\ S^M:SE^M\ \ii{Table}&
\ii{Rope}\ N^A\ \ii{Body} &
\ii{Muffler}\ O^A\ \ii{Drawer} \\
\ii{PhoneBook}\ O^A\ \ii{Table}&
\ii{Whistle}\ E^M\ \ii{Body}&
\ii{Balloon}\ E^M\ \ii{Body}.
\ea
\eeq{eq:suspect1}

From Suspect 2's statement, the following 3D-nCDC constraints are obtained:
\beq
\ba{llll}
\ii{Suitcase}\ S^M\ \ii{Table}&
\ii{Drug}\ O^A\ \ii{Teapoy} &
\ii{Knife}\ E^M\ \ii{Body}&
\ii{Knife}\ S^B\ \ii{Phone} .
\ea
\eeq{eq:suspect2}

Considering also the following commonsense knowledge about locations of objects:
\beq
\ba{lll}
\ii{default}\ \ii{Phone}\ O^A\ \ii{Table} &
\ii{default}\ \ii{Umbrella}\ O^M\ \ii{Hanger} &
\ii{default}\ \ii{Coat}\ O^M\ \ii{Hanger}.
\ea
\eeq{eq:commonsense}
the detective concludes that Suspect~1 is truthful whereas Suspect~2 is not.

The 3D-nCDC constraint network obtained from Suspect 2's statements~(\ref{eq:suspect2}), the digital evidence~(\ref{eq:digital}) and the commonsense knowledge (\ref{eq:commonsense}) is found inconsistent by \clingo, using the program $\Pi_{m,n,p}^{3}$. An explanation for this inconsistency is found by replacing the constraints (\ref{eq:violated}) in $\Pi_{m,n,p}^{3}$ with the weak constraints (\ref{eq:violated-weak}):
the atom $\ii{violated}(\ii{Knife},\ii{Phone})$ in the answer set indicates that the knife cannot be to the front and below of the phone.

\vspace{-2ex}
\subsection{Discussion}

We have presented three scenarios from different real-world applications. In each scenario,
the 3D-nCDC constraints are obtained from the qualitative directional constraints specified by agents.
The number of objects and constraints are reasonable from the perspectives of the relevant real-world applications.
Yet, for the purpose of investigating the scalability of our method, we have constructed larger scenarios with more number of objects and constraints by ``replicating'' the scenarios above multiple times. Instance $M1$ denotes the marine exploration scenario presented in Section~\ref{sec:marine}, with 5 spatial objects and 7 3D-nCDC constraints. Instances $M2$--$M4$ replicate this instance twice, three times, and four times, respectively.

We have also constructed some instances to investigate how computational performance changes when the instance is inconsistent. Instance $B1$ denotes the building design scenario presented in Section~\ref{sec:building}, with 6 spatial objects and 6 3D-nCDC constraints; it is inconsistent. Instance $B1'$ is a consistent instance obtained from $B1$ by dropping the violated 3D-nCDC constraint. Instances $B2$ and $B2'$ replicate Instances $B1$ and $B1'$ twice, respectively.
In addition, we have considered instances, $D1$ and $D2$, that describe the digital forensics scenarios presented in Section~\ref{sec:df-application-3d}, where the consistency of statements of Suspect 1 and 2 are checked, respectively.

We have measured the time and memory consumption for these consistency checking problem instances, on a workstation with 3.3GHz Intel Xeon W-2155 CPU and 32GB memory, using \clingo~5.3.0.   The results are shown in Table~\ref{tab:scenario-stats}.

We can observe from these results that, as the number of objects and constraints increase, the computation time and memory increase.

For instance, when the number of spatial variables and the number of 3D-nCDC constraints double, and the grid size increases more than $2^3$ times (from $M1$ to $M2$, $B1$ to $B2$, $B1'$ to $B2'$), the number of rules in the ground ASP program (as reported by \clingo) increases by almost 20 times. This is not surprising as the number of some rules (like (\ref{eq:project-mp-3d})) increase as many as $2^3{\times}2{=}16$ times.
Similarly, when the number of spatial variables and the number of 3D-nCDC constraints increase three times, and the grid size increases by at least $3^3$ times (from $M1$ to $M3$), the number of rules increases by almost 115 times. Such increase in the program size also causes an increase in the computation time and the memory consumption.

We also observe from Instances $B1$, $B2$ and $D2$ that the inconsistency of a network is determined in a longer time. This is not surprising either, since the search space is larger for these instances.

Note that due to Corollary~\ref{cor:correct-basic-3d} (obtained from Theorems~\ref{th:digital-equiv-3d} and~\ref{th:thm-correct-basic-3d}), our ASP method for consistency checking in 3D-nCDC is sound and complete. Therefore, in Table~\ref{tab:scenario-stats}, the solutions computed by \3ncdc for the benchmark instances are correct.

\begin{table}[t]
\caption{Experimental evaluations}
\label{tab:scenario-stats}
\begin{center}
\resizebox{0.9\textwidth}{!}{\begin{tabular}{ccccrrrr}
\hline\hline
Instance   &  $|V|$ &  $|C|$   &   Grid Size   &   \multicolumn{2}{c}{Grounding\&Total Time (sec)}      & Memory (GB)        &  \#Rules  \\
\hline\hline
$M1$  &  5   &   7  &  $9{\times}9{\times}9$   &   0.30   &   0.34  &  $<$0.01        &   241853 \\
$M2$  &  10   &   14  &  $19{\times}19{\times}19$   &   7.98   &   10.71  &  0.77       &   5050676 \\
$M3$  &  15   &   21  &  $29{\times}29{\times}29$   &   48.82   &   68.11  &  4.18       &   27826869 \\
$M4$  &  20   &   28  &  $39{\times}39{\times}39$   &   175.33   &   227.19  &  13.79        &   91678832 \\
$B1$ &  6   &   6  &  $11{\times}11{\times}11$   &   0.66   &   477.48  &  0.13   &     796379 \\
$B1'$  &  6   &   5  &  $11{\times}11{\times}11$   &   0.55   &   3.30  &  0.07   &     714772 \\
$B2$  &  12   &   12  &  $23{\times}23{\times}23$   &   16.27   &   $>$10000  &  2.57     &   15445966 \\
$B2'$  &  12   &   10  &  $23{\times}23{\times}23$   &   13.85   &   2174.47  &  1.48       &   13884200  \\
$D1$ &  16   &   15  &  $31{\times}31{\times}31$   &   282.64   &   4401.02  &  3.87      &    30577147 \\
$D2$ &  13   &   13  &  $25{\times}25{\times}25$   &   82.40   &   $>$10000  &  1.71     &   13253185 \\
\hline\hline
\end{tabular}}
\end{center}
\vspace{-2\baselineskip}
\end{table}

\section{Conclusion}    
\label{sec:conclude-3d}

We have introduced a general and provably correct framework (\3ncdc) for representing the cardinal directions between (dis)connected extended objects in 3D space, by means of 3D-nCDC constraints (including default 3D-nCDC constraints), and for reasoning about these relations using Answer Set Programming, based on a discretization of the space (preserving the meaning of cardinal directions in continuous space).

\3ncdc can be used to check the consistency of a set of 3D-nCDC constraints, infer unknown cardinal direction relations, and explain source of inconsistency. It can deal with the challenges of incomplete or uncertain knowledge as well as defaults about cardinal directions between objects, as often encountered in applications.

Allowing combinations of reasoning capabilities, \3ncdc provides a flexible environment and a computational tool for various real-world applications, as illustrated by some realistic scenarios in marine explorations with an underwater human-robot team, building design and regulation in architecture, and evidence-based digital forensics.


\vspace{-2ex}
\paragraph{\bf Acknowledgments}
We have benefited from useful discussions with Philippe Balbiani (on the use of ASP for qualitative reasoning about cardinal directions), Anthony Cohn, Volkan Patoglu and Subramanian Ramamoorthy (on applications of 3D-nCDC in robotics), Stefania Costantini (on applications of 3D-nCDC in digital forensics), and Mehdi Nourbakhsh (on applications of 3D-nCDC in building design). This work is partially supported by Cost Action CA17124.

\bibliographystyle{acmtrans}    


\appendix

\section{Answer Set Programming}\label{sec:asp-review-3d}

Answer Set Programming (ASP) is a knowledge representation and reasoning paradigm~\cite{MarekT99,Niemelae99,Lifschitz02}, based on answer set semantics~\cite{GelfondL88,gelfondL91}. It provides a formal framework for declaratively solving intractable problems, like consistency checking in CDC.  The idea of ASP is to model a problem by a set of logical formulas (called rules), so that its models (called answer sets) characterize the solutions of the problem. The models can be computed by ASP solvers, like \clingo~\cite{GebserKKOSS11}.

Let us briefly describe the syntax of programs and useful constructs used in the paper.
ASP provides logical formulas, called rules, of the form
$$ 
\ii{Head} \lar L_1, \dots, L_k, \no\ L_{k+1}, \dots, \no\ L_l
$$ 
where $l \geq k \geq 0$, \ii{Head} is a literal (i.e., an atom $A$ or its negation $\neg A$) or $\bot$, and each $L_i$ is a literal. A rule is called a \textit{constraint} if \ii{Head} is $\bot$, and a \textit{fact} if $l=0$. A set of rules is called a \textit{program}.

ASP can express both classical negation ($\neg$) and default negation ($\no$). For example, the following rule expresses that, normally, the elevator works fine (\ii{works}) unless stated or observed otherwise that it does not work ($\neg \ii{works}$):
$$
\ii{works} \lar \no\ \neg \ii{works} .
$$

ASP provides special constructs to express nondeterministic choices, cardinality constraints, and aggregates.  Programs using these constructs can be viewed as abbreviations for programs that consist of rules of the form above.

\textit{Choice rules} provide a concise representation for nondeterministic choices, and thus allow generation of answer sets. For instance, the answer sets for the choice rule
$$\{p_1,p_2,\dots,p_5\} \lar$$
are all subsets of the set $\{p_1,p_2,\dots,p_5\}$.

\textit{Cardinality expressions} are of the form $l \{L_1,\dots,L_k\} u$ where each $L_i$ is a literal and $l$ and $u$ are
nonnegative integers denoting the lower and upper bounds. Such an expression describes the subsets of the set
$\{L_1,\dots,L_k\}$ whose cardinalities are at least $l$ and at most~$u$.  Cardinality expressions can be used in heads of choice rules; then
they generate many answer sets whose cardinality is at least $l$ and at most $u$. For instance, the choice rule
$$ 
1\{p_1,p_2,\dots,p_5\}3 \lar
$$ 
allows nondeterministically selecting at least 1 and at most 3 elements of the set $\{p_1,p_2,\dots,p_5\}$ to be included in an answer set.
When a cardinality expression is in body of the rules, it imposes a cardinality constraint on the number of literals.
For instance, adding the following constraint
$$
\lar 2\{p_1,p_2,\dots,p_5\}
$$
to the choice rule above 
will impose a constraint on the choice rule, and thus only subsets of $\{p_1,p_2,\dots,p_5\}$ whose cardinality is exactly one will be generated.

\textit{Schematic variables} can be used to compactly describe a group of rules, or a set of literals in a choice rule. For instance, the cardinality expression
$1\{p_1,p_2,\dots,p_5\}3$ can be represented as $1\{p(i): \ii{index}(i)\}3$, along with a definition of $\ii{index}(i)$ to describe
the ranges of variables $i$: $\ii{index}(1..3)$. The following choice rule allows nondeterministically selecting at least 1 and at most 3 numbers $x$ for every set $u$:
$$
1\{\ii{select}(u,x) : \ii{num}(x)\} 3 \lar \ii{set}(u) .
$$

ASP also provides utilities to represent \textit{aggregates}. For instance, the following rule defines the smallest number, $N$, selected so far using the aggregate \ii{min}:
$$
\ii{smallest}(N) \lar \#\ii{min}\ \{x: \ii{select}(u,x), \ii{set}(u)\} \eqs N.
$$


\section{An Example on Consistency Checking with Projected Constraints}
\label{sec:proj-3d}

\citeN{li2009qualitative} propose to check the consistency of a set of 3D CDC constraints, by projecting each 3D directional relation onto $xy$, $yz$, $xz$ planes, and by expressing each 3D directional relation in terms of three 2D directional relations. With this method, a basic 3D-nCDC network $C$ can be transformed into three nCDC constraint networks $C_{xy}$, $C_{yz}$, $C_{xz}$ by projecting every basic 3D-nCDC constraint onto respective plane.
If $C$ is consistent on $\Regs$, then $C_{xy}$, $C_{yz}$, $C_{xz}$ are all consistent.
However, the reverse is not necessarily true.

Consider a 3D-nCDC network
$$
\ba l
C=\{u\ NE^A:NW^A:SW^B:SE^B\ t,\ v\ SW^A:SE^A:NE^B:NW^B\ t,\\
 u\ NE^A:NW^A:SW^A:SE^A:NE^B:NW^B:SW^B:SE^B\ v\}.
\ea
$$
This network is inconsistent on $\Regs$ because $u$ and $v$ occupy 4 tiles of $t$ according to the first two constraints but the last constraint imposes $u$ to occupy 8 tiles of $v$.

\begin{figure}[h]
    \centering
    \includegraphics[width=0.2\textwidth]{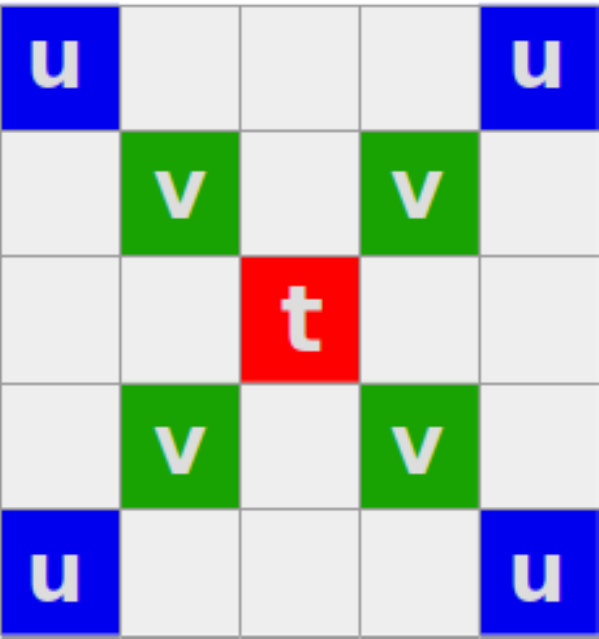}
    \caption{Solution for projected 2D networks}
    \label{fig:soln-2d}
\end{figure}

However, the projected 2D networks $C_{xy}$, $C_{yz}$, $C_{xz}$ are all consistent.
The projection of $C$ on \textit{xy,yz,xz} planes are the same:
$$C_{xy} \eqs C_{yz} \eqs C_{xz} \eqs \{u\ NE:NW:SW:SE\ t,\ v\ NE:NW:SW:SE\ t,\  u\ NE:NW:SW:SE\ v.$$
Note that $C_{xy}$, $C_{yz}$, $C_{xz}$ are consistent since the instantiation of objects in Figure~\ref{fig:soln-2d} is a solution to each of the three networks.

Therefore, projection of 3D-nCDC constraints on 2D space causes a loss of information.
This example illustrates why we consider consistency checking directly in 3D, instead of combining 2D consistency checking on projections of the network on \textit{xy,yz,xz} planes.


\section{An Example on Qualitative Reasoning about Inverses of Directional Relations}
\label{sec:inverse}

Qualitative directional relations in 3D are used in robotics. For instance, \citeN{zampogiannis2015learning} define six directional relations (i.e., \ii{left}, \ii{right}, \ii{front}, \ii{behind}, \ii{below}, \ii{above}) between point clouds in 3D by utilizing axis-aligned bounding boxes and 3D cones defined with respect to these boxes, for the purpose of grounding. \citeN{mota2018incrementally} consider a variation of \citeN{zampogiannis2015learning}'s definitions of the six directional relations where ``the spatial relation of an object with respect to a reference object is determined by the non-overlapping pyramid around the reference that has most of the point cloud of the object.''
However, such related work in robotics do not study reasoning problems, like consistency checking or inference of (missing) relations (e.g., compositions or inverses), in the spirit of the well-studied qualitative spatial reasoning calculi. The lack of formal studies on such reasoning problems might lead to incorrect conclusions and inferences.

%
%

For instance, consider two point clouds $A$ and $B$. Consider also directional relations as defined by \citeN{mota2018incrementally}. Suppose that we are given that $B$ is \ii{below}~$A$, and $A$ is to the \ii{right} of $B$. For simplicity of presentation, the projection of these relations on $xz$ plane are shown in Figure~\ref{fig:inverse2}. In this example, it will be incorrect to infer that $A$ is \ii{above} $B$ according to \citeN{mota2018incrementally}'s ASP rule:
$$\ii{holds}(\ii{above}(A,B),I) \lar \ii{holds}(\ii{below}(B,A),I).$$
This ASP rule (and the ASP program that includes this rule) is not correct from the qualitative spatial reasoning point of view, with respect to the definitions of directional relations~\cite{mota2018incrementally}.

\begin{figure}[t]   
    \centering
    \includegraphics[width=0.5\textwidth]{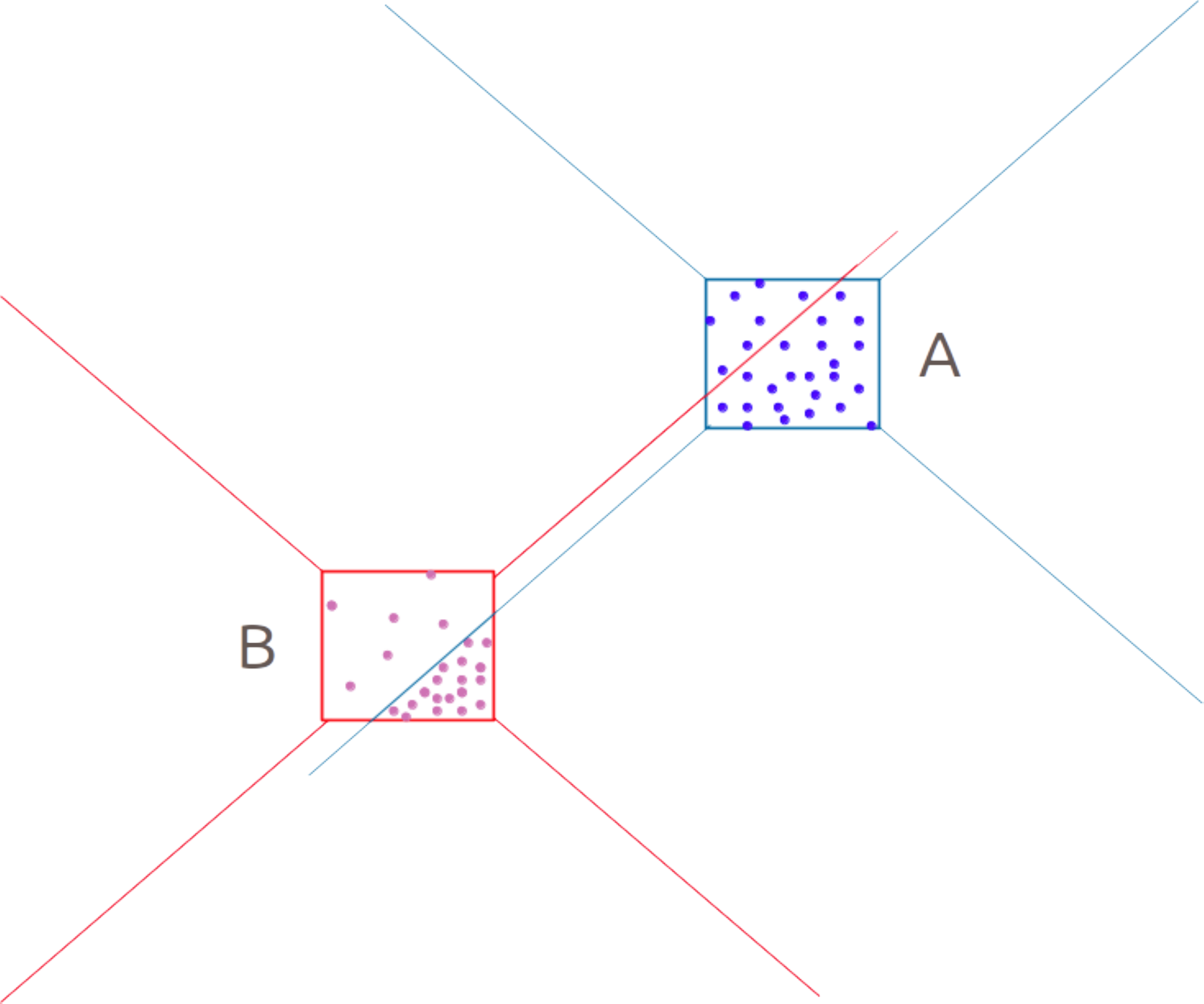}
    \vspace{-\baselineskip}
    \caption{Another example: The projection of objects $A$ and $B$ on $xz$ plane.}
    \label{fig:inverse2}
\end{figure}

This example illustrates that, although introducing qualitative spatial relations may be sufficient for low-level tasks in robotics like grounding, further formal studies are required about reasoning problems, like consistency checking or inference of relations, for correct high-level reasoning in robotics. Furthermore, the correctness of formulations over qualitative spatial relations also needs to be investigated to prevent unsound inferences. In that sense, \3ncdc provides a provably correct method and tool for reasoning about 3D cardinal directions, that robotics studies can benefit from.


\section{ASP Representation of the Example Scenarios}

\subsection{Marine Exploration with Underwater Robots}

{\small
\begin{verbatim}
% Spatial objects are possibly disconnected objects in 3D space
% 1:sedimentary rock   2:fragmented marsh   3:volcanic vent
% 4:kelp forest   5:fungi culture

object(1..5).

% 27 tiles of 3D-nCDC
alltiles(swm).  alltiles(sm).  alltiles(sem).  alltiles(wm).
alltiles(om).   alltiles(em).  alltiles(nwm).  alltiles(nm).
alltiles(nem).  alltiles(swb). alltiles(sb).   alltiles(seb).
alltiles(wb).   alltiles(ob).  alltiles(eb).   alltiles(nwb).
alltiles(nb).   alltiles(neb). alltiles(swa).  alltiles(sa).
alltiles(sea).  alltiles(wa).  alltiles(oa).   alltiles(ea).
alltiles(nwa).  alltiles(na).  alltiles(nea).

% Define network of constraints in 3D-nCDC
relation(2,1,swb).  relation(2,1,seb).
relation(3,1,ea).  relation(3,1,sea).
relation(3,2,nea).
relation(4,3,nb).  relation(4,3,seb).
relation(5,4,sm).
relation(5,2,eb).

% Infer Unknown Relation
toinfer(5,1).
\end{verbatim}
}

\subsection{Building Design and Regulation}

{\small
\begin{verbatim}
% Spatial objects are connected objects in 3D space
% 1:entrance   2:electric panel   3:system room
% 4:heating unit   5:director office   6:secretary room

object(1..6).

% 27 tiles of 3D-nCDC
alltiles(swm).  alltiles(sm).  alltiles(sem).  alltiles(wm).
alltiles(om).   alltiles(em).  alltiles(nwm).  alltiles(nm).
alltiles(nem).  alltiles(swb). alltiles(sb).   alltiles(seb).
alltiles(wb).   alltiles(ob).  alltiles(eb).   alltiles(nwb).
alltiles(nb).   alltiles(neb). alltiles(swa).  alltiles(sa).
alltiles(sea).  alltiles(wa).  alltiles(oa).   alltiles(ea).
alltiles(nwa).  alltiles(na).  alltiles(nea).

% Define network of constraints in 3D-nCDC
disjrelation(2,1,1,nm).  disjrelation(2,1,2,nb).
disjrelation(3,2,1,om).  disjrelation(3,2,2,ob).
disjrelation(3,2,3,oa).
relation(5,1,oa).
relation(3,5,wm).

%  Commonsense Knowledge
defaultrelation(4,1,swb).
defaultrelation(6,5,em).

% Mandatory Constraints
mandatory(2,1).
\end{verbatim}
}

\subsection{Evidence-Based Digital Forensics}

{\small
\begin{verbatim}
% Spatial objects are connected objects in 3D space
% 1:Body  2:Table  3:Chair  4:Teapoy  5:Drawer
% 6:Sofa  7:Suitcase  8:Hanger  9:Coat  10:Hat
% 11:Muffler  12:Phonebook  13:Knife  14:Drug
% 15:Rope  16:Ballon  17:Whistle 18:bed 19:phone
% 20:umbrella

object(1..20).

% 27 tiles of 3D-nCDC
alltiles(swm).  alltiles(sm).  alltiles(sem).  alltiles(wm).
alltiles(om).   alltiles(em).  alltiles(nwm).  alltiles(nm).
alltiles(nem).  alltiles(swb). alltiles(sb).   alltiles(seb).
alltiles(wb).   alltiles(ob).  alltiles(eb).   alltiles(nwb).
alltiles(nb).   alltiles(neb). alltiles(swa).  alltiles(sa).
alltiles(sea).  alltiles(wa).  alltiles(oa).   alltiles(ea).
alltiles(nwa).  alltiles(na).  alltiles(nea).

% Murder Data
relation(1,2,sm).  relation(1,2,sem).
relation(1,4,nm).  relation(1,4,nem).
disjrelation(7,2,1,sm).  disjrelation(7,2,2,swm).
relation(19,2,oa).
relation(6,18,sem).
relation(4,6,em).
relation(9,8,om).

% Commonsense Knowledge:
defaultrelation(19,2,oa).
defaultrelation(20,8,om).
defaultrelation(9,8,om).

% Suspect 1 Statement
relation(15,1,na).
relation(11,5,oa).
relation(12,2,oa).
relation(17,1,em).
relation(16,1,em).

% Suspect 2 Statement
relation(7,2,sm).
relation(14,4,oa).
relation(13,1,em).
relation(13,19,sb).
\end{verbatim}
}

\section{Proof of Theorem~\ref{th:complexity-3d}}

%

Consider two cases: $C$ is an incomplete basic 3D-nCDC network, or $C$ includes disjunctive 3D-nCDC constraints.

\paragraph{\bf Case 1: $C$ is an incomplete basic 3D-nCDC network.}
We prove NP-membership and NP-hardness of $I\eqs(C,V,D,Q)$ as follows.

\smallskip\noindent{NP-membership:} $C$ includes at most $|V|(|V|-1)$ constraints. Testing a 3D-nCDC constraint between a pair of objects takes  $O(1)$ time. So, given a candidate solution $A = (a_i)^l_{i=1}$ of $I$, it takes
$O(|V|^2)$ time to verify all constraints in $C$. Hence, $I \ins NP$.

\smallskip\noindent{NP-hardness:} We reduce the 2D CDC consistency checking problem to the 3D CDC consistency checking problem.

Note that, according to Theorem 5.8 of~\citeN{Liuthesis2013}, consistency checking of an incomplete basic network of 2D CDC constraints over set of (possibly) disconnected objects in $\mathbb{R}^2$ is NP-complete.

Take an arbitrary instance $I'\eqs(C',V,D',Q')$ of 2D CDC consistency checking problem, where the network $C'$ consists of basic 2D CDC constraints, $D'$ is the set of (possibly) disconnected objects in $\mathbb{R}^2$, and $Q'$ is the set of all basic 2D CDC relations.
We reduce $I'$ to the following specific instance $I\eqs(C,V,D,Q)$ of 3D CDC consistency checking problem.
The set $V$ of spatial variables stays the same. For every basic 2D CDC constraint $u\ R_1:...:R_k \ v$ in $C'$, we insert the corresponding basic 3D-nCDC constraint $ u\ R^{M}_1:...:R^{M}_k \ v $ into $C$.
Namely, the 2D constraints are assumed to be on the middle level of $z$ axis and thereby transformed into 3D constraints.
Since a basic constraint in $C'$ can have at most 9 tiles, this reduction takes $O(|C'|)$ time, which is polynomial in input size.

Next, we prove that this reduction is correct. For this, we show that the answer of $I'$ is Yes if and only if the answer of $I$ is Yes.
First, suppose the answer of $I'$ is Yes and there exists a solution $A' = (a'_i)^l_{i=1}$ of $I'$.
That is, $a'_i, a'_j \ins A'$ satisfies the basic 2D CDC constraint $u_i\ R_{ij,1}:...:R_{ij,k} \ u_j$ in $C'$.
Using $A'$, we construct another instantiation $A \eqs (a_i)^l_{i=1}$ which is a solution of $I$:
We stretch every planar object $a'_i \ins A'$ along $z$ dimension by an amount $\kappa \gts 0$ in a manner that all objects accommodate the range [0,$\kappa$] on $z$ axis. With this method, we create 3D objects $A \eqs (a_i)^l_{i=1}$ from 2D objects $(a'_i)^l_{i=1}$ such that the projection of each $a_i$ on the $xy$ plane is equal to $a'_i$ ($1 \leqs i \leqs l$). Since all objects in $A$ are aligned on the $z$ axis, a pair $(a_i, a_j)$ in $A$ satisfies the 3D CDC constraint $u_i\ R^{M}_{ij,1}:...:R^{M}_{ij,k} \ u_j$ in $C$.
Thus, $A$ satisfies $C$ and the answer of $I$ is Yes.
For the reverse direction, suppose the answer of $I$ is Yes and there exists a solution $A = (a_i)^l_{i=1}$ of $I$.
Then $A$ satisfies every 3D CDC constraint $u_i\ R^{M}_{ij,1}:...:R^{M}_{ij,k} \ u_j$ in $C$.
We construct a solution $A' \eqs (a'_i)^l_{i=1}$ of $I'$ using $A$:
we project each $a_i \ins A$, $1 \leqs i \leqs l$ onto $xy$ plane and designate the projection as a planar object $a'_i$. This way, a 2D instantiation $A' = (a'_i)^l_{i=1}$ is formed.
Note that $A'$ satisfies every 2D CDC constraint $u_i\ R_{ij,1}:...:R_{ij,k} \ u_j$ in $C'$ by construction.
Consequently, $A'$ is a solution of $I'$ and the answer of $I'$ is Yes.
This means $I'$ and $I$ have the same answers, and thus concludes the proof of NP-hardness of $I$.

\paragraph{\bf Case 2: $C$ includes disjunctive 3D-nCDC constraints.}
The proof of NP-membership of $I$ is the same as in the first case.
To prove NP-hardness, we reduce the 2D CDC consistency checking problem to the 3D CDC consistency checking problem.

Note that consistency checking of a network of (possibly disjunctive) 2D CDC constraints over set of (possibly) disconnected objects in $\mathbb{R}^2$  is NP-complete by Theorem 6 of \cite{SkiadKoub2005}.

Take an arbitrary instance $I'\eqs(C',V,D',Q')$ of 2D CDC consistency checking problem,  where $C'$ consists of basic and disjunctive 2D CDC constraints, $D'$ is the set of (possibly) disconnected objects in $\mathbb{R}^2$, and $Q'$ is the set of all 2D CDC relations.
We reduce $I'$ to the following specific instance $I\eqs(C,V,D,Q)$ of 3D CDC consistency checking problem.
For a basic 2D CDC constraint $u\ R_1:...:R_k \ v$ in $C'$, we insert the basic 3D-nCDC constraint $ u\ R^{M}_1:...:R^{M}_k \ v $ into $C$. For a disjunctive 2D CDC constraint $u\ \{ \delta_1, ..., \delta_k \} \ v$, we mark tiles of every disjunct (basic relation) $\delta_i$ on middle level of $z$ axis and insert the new disjunctive 3D-nCDC constraint into $C$.
Since a disjunctive constraint in $C'$ can have at most $2^9-1$ disjuncts and a basic 2D CDC relation can have at most 9 tiles, running time of this reduction is $O(|C'|)$, which is polynomial in input size.

Next, we prove the correctness of this reduction.
Suppose the answer of $I'$ is Yes and there exists a solution $A' = (a'_i)^l_{i=1}$ of $I'$.
The instantiation $A'$ satisfies every basic or disjunctive 2D CDC constraint in $C'$.
We construct an instantiation $A \eqs (a_i)^l_{i=1}$ of 3D objects, similar to the construction in the first part of the theorem.
Each planar object $a'_i \ins A'$ is elongated along $z$ dimension by an amount $\kappa \gts 0$ to form a 3D object $a_i$ which accommodates the range [0,$\kappa$] on $z$ axis.
Note that $a'_i, a'_j \ins A'$ satisfies a basic 2D CDC constraint or a disjunct of a disjunctive constraint $u_i\ R_{ij,1}:...:R_{ij,k} \ u_j$ in $C'$. Therefore, the pair $a_i, a_j$ in $A$ satisfies the corresponding basic 3D-nCDC constraint or disjunct $u_i\ R^{M}_{ij,1}:...:R^{M}_{ij,k} \ u_j$ in $C$. Thus $A$ satisfies $C$ and the answer of $I$ is Yes.

For the reverse direction, suppose the answer of $I$ is Yes and there exists a solution $A = (a_i)^l_{i=1}$ of $I$.
Note that $A$ satisfies every basic or disjunctive 3D CDC constraint in $C$.
A solution $A' \eqs (a'_i)^l_{i=1}$ of $I'$ is formed using $A$ as follows.
Each object $a_i \ins A$, $1 \leqs i \leqs l$ is projected onto $xy$ plane and the projection is designated as a planar object $a'_i$, similar to first part.
Since $a_i, a_j \ins A$ satisfies a basic 3D-nCDC constraint or a disjunct of a disjunctive constraint $u_i\ R^{M}_{ij,1}:...:R^{M}_{ij,k} \ u_j$ in $C$, the pair $(a'_i, a'_j)$ in $A'$ satisfies the corresponding basic 2D CDC constraint or a disjunct $u_i\ R_{ij,1}:...:R_{ij,k} \ u_j$ in $C'$.
This way, a 2D instantiation $A' = (a'_i)^l_{i=1}$ that satisfies $C'$ is constructed. Hence, the answer of $I'$ is Yes.
We conclude that answers of $I'$ and $I$ are the same, and thus the proof of NP-hardness of $I$.



\section{Proof of Theorem~\ref{th:digital-equiv-3d}}

%
%

We show that the answer of $I\eqs(C,V,D,Q)$ is Yes if and only if the answer of $I_{m,n,p}\eqs(C,V,D_{m,n,p},Q)$, $m,n,p \geqs 2|V|-1$ is Yes.

\paragraph{Right to left.} Suppose that the answer of the discretized problem $I_{m,n,p}$ is Yes.
Using a solution $A' = (a'_i)^l_{i=1}$ of $I_{m,n,p}$, we construct a solution $A = (a_i)^l_{i=1}$ of $I$ as follows:
Origin of the prism (3-dimensional grid) is viewed as the origin of $\mathbb{R}^3$, grid cells will be converted into closed cubes in $\mathbb{R}^3$ and the object $a_i$ on the Euclidean space is equal to the union of cubes occupied by $a'_i$.
By construction, $A = (a_i)^l_{i=1}$ satisfies $C$ and the answer of $I$ is Yes.

\paragraph{Left to right.} Suppose that the answer of $I$ is Yes.
Take any solution $A = (a_i)^l_{i=1}$ of $I$ in $\Regs$. Note that objects might be disconnected.
$A$ satisfies basic 3D-nCDC constraints in $C$.
Since objects are compact sets in $\mathbb{R}^3$, they are bounded.

We identify axes-aligned minimum bounding box of every object in $A$ and then order their bounds.
Let $ B_x \eqs ( \inf_x(a_i), \sup_x(a_i) \: : \:  1\leqs i \leqs l ) $, $ B_y \eqs ( \inf_y(a_i), \sup_y(a_i) \: : \:  1\leqs i \leqs l ) $,  \\
$ B_z \eqs ( \inf_z(a_i), \sup_z(a_i) \: : \:  1\leqs i \leqs l ) $
be ascending ordered list of infimum and supremums of these objects over respective axis.
Note that numbers in an ordered list may not be all distinct because infimum/supremum of an object might coincide with infimum/supremum of another object in $A$.
These bounds in $B_x$, $B_y$, $B_z$ partition the Euclidean space into cubic zones.

We build a 3-dimensional grid (prism) of size $m \times n \times p$ using zones created by $B_x$, $B_y$, $B_z$, as below.
The cubic zones whose coordinates are less than the minimum element or greater than the maximum element of respective list ($B_x$, $B_y$ or $B_z$) are omitted so that we restrict attention to only the zones which might be occupied by some object in $A$.

We will construct an instantiation $A' = (a'_i)^l_{i=1}$ on the prism which satisfies $C$.
The indices of the prism on x axis corresponds to the respective element of $ B_x$ in ascending order, namely $i^{th}$ index of the prism on $x$ axis is the $i^{th}$ element of $B_x$.
In case infimum/supremum of multiple objects coincide, they correspond to the same index on the prism.
An analogous indexing scheme is applied to the $y$ and $z$ axes.
Observe that there is a 1-1 correspondence between the abovementioned cubic zones and the grid cells.

We form discrete objects $(a'_i)^l_{i=1}$ over the prism as follows:
Objects on the prism are set of cells and they can be disconnected.
If an object $a_i$ occupies a positive volume on a cubic zone, we assign the corresponding grid cell to $a'_i$.
The same cell can be allocated to multiple objects.
Recall that zero volume components (i.e. individual points, lines, surfaces) do not alter 3D-nCDC relations.
This manner ordering of infimum and supremums of objects  $(a'_i)^l_{i=1}$ on the grid are the same of original objects $(a_i)^l_{i=1}$ on the Euclidean space. Consequently, orientation of the minimum bounding box of the objects and occupied tiles stay the same.
Therefore 3D-nCDC relations between $(a'_{i}, a'_{j})$ are the same as $(a_{i}, a_{j})$.
Hence, $A' = (a'_i)^l_{i=1}$ also satisfies $C$ and the answer of $I_{m,n,p}$ is Yes.

The size of the prism (number of cells) on each axis is equal to the number of indices on that axis, less 1.
Since there can be at most $2|V|$ distinct elements in $B_x$, $B_y$, $B_z$, the prism can have a maximum of $2|V|-1$ cells on each axes.
Namely, a solution of $I$ can be constructed on a grid of size $m \eqs n \eqs p \eqs 2|V|-1$ or larger.



\section{Proof of Theorem~\ref{th:thm-correct-basic-3d}}

The proof of Theorems~\ref{th:thm-correct-basic-3d},~\ref{th:asp-correct-disj-3d} uses the following results in~\cite{Erdogan2004}.

\vspace{0.15in}

\noindent{\bf Splitting Set Theorem~\cite{Erdogan2004}}. Let $U$ be a splitting set for a program $\Pi$. A consistent set of literals is an answer set for $\Pi$ if it can be written as $X\cup Y$ where $X$ is an answer set for $b_U(\Pi)$ and $Y$ is an answer set for $e_U(\Pi \setminus b_U(\Pi), X)$.

Intuitively, the bottom part $b_U(\Pi)$ of a program $\Pi$ consists of the rules whose literals are contained in the splitting set $U$. Once an answer set $X$ for the bottom part is computed, it is ``propagated’’ to the rest of the program (called the top part) and the answer set $Y$ is computed for the top part. The theorem ensures that $X \cup Y$ is an answer set for the whole program.

\vspace{0.15in}

\noindent{\bf Proposition~2 of~\cite{Erdogan2004}}. For any program $\Pi$ and formula $F$, a set $Z$ of literals is an
answer set for $\Pi \cup \{\lar F\}$ if $Z$ is an answer set for $\Pi$ and does not satisfy $F$.

Intuitively, Proposition~2 of~\cite{Erdogan2004} express that adding constraints to an ASP program eliminates its answer sets that violate these constraints.

\vspace{0.28in}
%
%

\begin{proof}[Proof of Theorem~\ref{th:thm-correct-basic-3d}]
\label{prf:thm-correct-basic-3d-a}
The correctness proof for the whole program consists of three parts, considering the rules for the input network, the rules for generating minimum bounding box and instantiation of objects, and the rules for 3D-nCDC constraints. It is followed by the uniqueness proof.

\paragraph{\bf Correctness proof}
\medskip\noindent{\underline{Rules for the input network:}}
Every answer set for the subprogram~$(\ref{eq:rel-3d})$ characterizes the basic 3D-nCDC constraints as the input of the consistency problem, and every answer set for the subprogram~$(\ref{eq:existrel-3d})$ shows pairs of variables for which a constraint exists in the network with $\ii{existrel}(u,v)$ atoms.

\medskip\noindent{\underline{Rules for MBB and instantiation of objects:}}
We first consider the subprogram $\Pi^{1,a}_{m,n,p}$ that consists of the set $F_B$ of facts $(\ref{eq:rel-3d})$, the rules of the form $(\ref{eq:existrel-3d})$, the rule $(\ref{eq:generate-infsup-3d})$, and rules analogous to $(\ref{eq:generate-infsup-3d})$ that describe $\ii{inf}_{y}(u, \underline{y})$, $\ii{sup}_{y}(u, \overline{y})$, $\ii{inf}_{z}(u, \underline{z})$, $\ii{sup}_{z}(u, \overline{z})$.
We apply the splitting set theorem~\cite{Erdogan2004} to $\Pi^{1,a}_{m,n,p}$:
The set of all $\ii{inf}_{x}(u, \underline{x})$, $\ii{sup}_{x}(u, \overline{x})$, $\ii{inf}_{y}(u, \underline{y})$, $\ii{sup}_{y}(u, \overline{y})$, $\ii{inf}_{z}(u, \underline{z})$, $\ii{sup}_{z}(u, \overline{z})$ atoms is a splitting set for $\Pi^{1,a}_{m,n,p}$.
The bottom part is the rule $(\ref{eq:generate-infsup-3d})$ and rules analogous to $(\ref{eq:generate-infsup-3d})$.
An answer set $Y_1$ for the bottom part describes a possible choice of minimum bounding box of each spatial variable in terms of $\ii{inf}_{x}(u, \underline{x})$, $\ii{sup}_{x}(u, \overline{x})$, $\ii{inf}_{y}(u, \underline{y})$, $\ii{sup}_{y}(u, \overline{y})$, $\ii{inf}_{z}(u, \underline{z})$, $\ii{sup}_{z}(u, \overline{z})$ atoms.
An answer set $Y_2$ for the top part $(\ref{eq:rel-3d}) \cup (\ref{eq:existrel-3d})$  evaluated with respect to $Y_1$ describes 3D-nCDC constraints in $C$, pair of objects which have a constraint in $C$; and $Y_1 \cup Y_2$ is an answer set for $\Pi^{1,a}_{m,n,p}$.

The rule $(\ref{eq:infsup-ineq-3d})$ and analogous rules for $\ii{inf}_{y}(u, \underline{y})$, $\ii{sup}_{y}(u, \overline{y})$, $\ii{inf}_{z}(u, \underline{z})$, $\ii{sup}_{z}(u, \overline{z})$ insist on chosen infimum value to be less than or equal to supremum.
Proposition~2 of~\cite{Erdogan2004} implies that adding rule $(\ref{eq:infsup-ineq-3d})$ and rules analogous to $(\ref{eq:infsup-ineq-3d})$, the answer sets for $\Pi^{1,a}_{m,n,p}$ that do not have a valid minimum bounding box of an object are eliminated.
Thereby, answer sets of subprogram
$\Pi^{1,b}_{m,n,p}$ which consists of $\Pi^{1,a}_{m,n,p}$, the rule $(\ref{eq:infsup-ineq-3d})$ and analogous rules to $(\ref{eq:infsup-ineq-3d})$ represent 3D-nCDC constraints in $C$ and a valid choice of $\ii{inf}_{x}(u, \underline{x})$, $\ii{sup}_{x}(u, \overline{x})$, $\ii{inf}_{y}(u, \underline{y})$, $\ii{sup}_{y}(u, \overline{y})$, $\ii{inf}_{z}(u, \underline{z})$, $\ii{sup}_{z}(u, \overline{z})$ atoms.

Now we examine the subprogram $\Pi^{1,c}_{m,n,p} \eqs \Pi^{1,b}_{m,n,p} \cup (\ref{eq:generate-3d})$.
According to the splitting set theorem, the set $F_B$ of facts in $(\ref{eq:rel-3d})$
and the set of all $\ii{existrel}(u,v)$, $\ii{inf}_{x}(u, \underline{x})$, $\ii{sup}_{x}(u, \overline{x})$, $\ii{inf}_{y}(u, \underline{y})$, $\ii{sup}_{y}(u, \overline{y})$, $\ii{inf}_{z}(u, \underline{z})$, $\ii{sup}_{z}(u, \overline{z})$ atoms is a splitting set for $\Pi^{1,c}_{m,n,p}$.
The bottom part of $\Pi^{1,c}_{m,n,p}$ is $\Pi^{1,b}_{m,n,p}$ and the top part is $(\ref{eq:generate-3d})$.
An answer set $Y_3$ for the bottom part describes 3D-nCDC constraints in $C$ and a valid choice of bounds $\ii{inf}_{x}(u, \underline{x})$, $\ii{sup}_{x}(u, \overline{x})$, $\ii{inf}_{y}(u, \underline{y})$, $\ii{sup}_{y}(u, \overline{y})$, $\ii{inf}_{z}(u, \underline{z})$, $\ii{sup}_{z}(u, \overline{z})$ for every spatial variable.
An answer set $Y_4$ for the top part evaluated with respect to $Y_3$ describes a possible instantiation $A=(a_{i})^{l}_{i=1}$ of objects to variables in $V$ and $Y_3 \cup Y_4$ is an answer set for $\Pi^{1,c}_{m,n,p}$.

In the next step, we add rules of the form $(\ref{eq:project-mp-3d})$  and rules analogous to $(\ref{eq:project-mp-3d})$ for $y$, $z$ axes into $\Pi^{1,c}_{m,n,p}$ to form subprogram $\Pi^{1,d}_{m,n,p}$.
The set $F_B$ of facts in $(\ref{eq:rel-3d})$
and the set of all $\ii{occ}(u,x,y,z)$, $\ii{existrel}(u,v)$, $\ii{inf}_{x}(u, \underline{x})$, $\ii{sup}_{x}(u, \overline{x})$, $\ii{inf}_{y}(u, \underline{y})$, $\ii{sup}_{y}(u, \overline{y})$, $\ii{inf}_{z}(u, \underline{z})$, $\ii{sup}_{z}(u, \overline{z})$ atoms is a splitting set for $\Pi^{1,d}_{m,n,p}$.
The bottom part of $\Pi^{1,d}_{m,n,p}$ is $\Pi^{1,c}_{m,n,p}$ and an answer set $Y_5$ for the bottom part describes 3D-nCDC constraints in $C$, a valid choice of bounds $\ii{inf}_{x}(u, \underline{x})$, $\ii{sup}_{x}(u, \overline{x})$, $\ii{inf}_{y}(u, \underline{y})$, $\ii{sup}_{y}(u, \overline{y})$, $\ii{inf}_{z}(u, \underline{z})$, $\ii{sup}_{z}(u, \overline{z})$ for every variable $u\in V$ and an instantiation of objects to every variable.
The top part is the rule $(\ref{eq:project-mp-3d})$ and rules analogous to $(\ref{eq:project-mp-3d})$ for $\ii{yocc}(u,y)$, $\ii{zocc}(u,z)$ atoms.
An answer set $Y_6$ for the top part evaluated with respect to $Y_5$ indicates the projection of each generated object over $x,y,z$ axes with $\ii{xocc}(u,x)$, $\ii{yocc}(u,y)$, $\ii{zocc}(u,z)$ atoms; and $Y_5 \cup Y_6$ is an answer set for $\Pi^{1,d}_{m,n,p}$.

The rules $(\ref{eq:cells-inside-mbr-3d})$  and analogous rules for $y,z$ axes impose cells of every object are generated inside its minimum bounding box. Rule $(\ref{eq:infsup-existcell-3d})$ and analogous rules impose at least one cell has been generated on infimum and supremum on every axes to make sure that correct values have been chosen.
Inserting the rules $(\ref{eq:cells-inside-mbr-3d})$ and $(\ref{eq:infsup-existcell-3d})$, and analogous rules for $y$, $z$ axes into $\Pi^{1,d}_{m,n,p}$ eliminates answer sets of $\Pi^{1,d}_{m,n,p}$ that do not obey these criteria.
Thus, we form the subprogram $\Pi^{1,e}_{m,n,p}$ which is composed of $\Pi^{1,d}_{m,n,p}$, the rules $(\ref{eq:cells-inside-mbr-3d})$, $(\ref{eq:infsup-existcell-3d})$ and the rules analogous to $(\ref{eq:cells-inside-mbr-3d})$, $(\ref{eq:infsup-existcell-3d})$ for $\ii{inf}_{y}(u, \underline{y})$, $\ii{sup}_{y}(u, \overline{y})$, $\ii{inf}_{z}(u, \underline{z})$, $\ii{sup}_{z}(u, \overline{z})$.
An answer set $Y_7$ of subprogram $\Pi^{1,e}_{m,n,p}$ represents 3D-nCDC constraints in $C$, a possible instantiation $A=(a_{i})^{l}_{i=1}$ of variables in $V$ and the correct bounds $\ii{inf}_{x}(u, \underline{x})$, $\ii{sup}_{x}(u, \overline{x})$, $\ii{inf}_{y}(u, \underline{y})$, $\ii{sup}_{y}(u, \overline{y})$, $\ii{inf}_{z}(u, \underline{z})$, $\ii{sup}_{z}(u, \overline{z})$ of objects.

\medskip\noindent{\underline{Rules for 3D-nCDC constraints:}}
Rules $(\ref{eq:c1-3d})$ and rules $(\ref{eq:c2-3d})$ find out whether the instantiation of objects $A$ violates 3D CDC tile constraints (C1) and (C2), respectively.
We consider the subprogram $\Pi^{1,f}_{m,n,p} \eqs \Pi^{1,e}_{m,n,p} \cup (\ref{eq:c1-3d}) \cup (\ref{eq:c2-3d})$.
A splitting set for $\Pi^{1,f}_{m,n,p}$ is the set $F_B$ of facts in $(\ref{eq:rel-3d})$
and the set of all $\ii{occ}(u,x,y,z)$, $\ii{existrel}(u,v)$, $\ii{inf}_{x}(u, \underline{x})$, $\ii{sup}_{x}(u, \overline{x})$, $\ii{inf}_{y}(u, \underline{y})$, $\ii{sup}_{y}(u, \overline{y})$, $\ii{inf}_{z}(u, \underline{z})$, $\ii{sup}_{z}(u, \overline{z})$, $\ii{xocc}(u,x)$, $\ii{yocc}(u,y)$, $\ii{zocc}(u,z)$ atoms.
The bottom part is $\Pi^{1,e}_{m,n,p}$ and the top part is $(\ref{eq:c1-3d}) \cup (\ref{eq:c2-3d})$.
An answer set $Y_{8}$ for the top part evaluated with respect to an answer set $Y_{7}$ of the bottom part indicates whether the instantiation $A$ violates conditions (C1),(C2) with $\ii{violated}(u,v)$ atoms and $Y_{7} \cup Y_{8}$ is an answer set for $\Pi^{1,f}_{m,n,p}$.

The rule $(\ref{eq:violated})$ prohibits 3D CDC constraints in (C1) and (C2) to be violated for any tile $\ii{rel}(u,v,R)$.
By adding rule $(\ref{eq:violated})$ into $\Pi^{1,f}_{m,n,p}$, the answer sets of $\Pi^{1,f}_{m,n,p}$ that do not satisfy 3D CDC constraints for $\ii{rel}(u,v,R)$ are eliminated.
Answer sets of the subprogram $\Pi^{1,g}_{m,n,p} \eqs \Pi^{1,f}_{m,n,p} \cup (\ref{eq:violated})$ represent 3D-nCDC constraints in $C$, a possible instantiation $A=(a_{i})^{l}_{i=1}$ of variables in $V$ that satisfy conditions (C1), (C2) and the minimum bounding box of the instantiated objects.

Note that $\Pi^{1}_{m,n,p} \eqs \Pi^{1,g}_{m,n,p} \eqs \Pi^{1,f}_{m,n,p} \cup (\ref{eq:violated})$.
If $Z$ is an answer set for $\Pi^{1}_{m,n,p}$,  $Z \cap \mathcal{O}_{m,n,p}$ characterize an instantiation $A$ of objects in $D_{m,n,p}$ to variables in $V$ that satisfy 3D-nCDC constraints in $C$.
Then, $X$ is a solution of $I_{m,n,p}$ if and only if $X$ can be characterized as $Z \cap \mathcal{O}_{m,n,p}$ for some answer set $Z$ of $\Pi^{1}_{m,n,p}$.

\paragraph{\bf Uniqueness proof}
To prove uniqueness of representation, suppose that another answer set $Z'$ for $\Pi^{1}_{m,n,p}$ also characterizes $X$ and $Z'\neq Z$.
$Z'$ must include precisely the same $\ii{occ}(u,x,y,z)$ atoms as $Z$, otherwise $Z'$ does not characterize $X$. Consequently projected coordinates $\ii{xocc}(u,x)$, $\ii{yocc}(u,y)$, $\ii{zocc}(u,z)$ atoms are the same for $Z'$ and $Z$. Because minimum bounding box of an object is unique, $\ii{inf}_{x}(u, \underline{x})$, $\ii{sup}_{x}(u, \overline{x})$, $\ii{inf}_{y}(u, \underline{y})$, $\ii{sup}_{y}(u, \overline{y})$, $\ii{inf}_{z}(u, \underline{z})$, $\ii{sup}_{z}(u, \overline{z})$ atoms in $Z$ and $Z'$ are also identical. Since all atoms in the two sets coincide, $Z'\eqs Z$.

\end{proof}


\section{Proof of Theorem~\ref{th:asp-correct-disj-3d}}

%
%
%

The proof is similar to the proof of Theorem~\ref{th:thm-correct-basic-3d} and consists of two parts: Correctness proof (considering the rules for the input network, the rules for generating minimum bounding box and instantiation of objects, the rules for 3D-nCDC constraints) and uniqueness proof.

\paragraph{\bf Correctness proof}

\medskip\noindent{\underline{Rules for the Input Network:}}
The answer set for the program~$(\ref{eq:cdc-disj-rel-3d})$ characterizes the disjunctive 3D-nCDC constraints in $C$.
Consider the subprogram $\Pi^{2,a}_{m,n,p}$ which consists of the set $F_V$ of facts in $(\ref{eq:cdc-disj-rel-3d})$ and the rules $(\ref{eq:disj-choose-3d})$, $(\ref{eq:rel-disj-3d})$ copied below:
\begin{align}
1\{\ii{chosen}(u,v,i): 1\leq i \leq o \}1 \lar  \nonumber \\ 
\ii{rel}(u,v,R) \lar \ii{chosen}(u,v,i), \: \ii{disjrel}(u,v,i,R).  \nonumber  
\end{align}
We apply the splitting set theorem~\cite{Erdogan2004} to $\Pi^{2,a}_{m,n,p}$:
The set of all possible $\ii{chosen}(u,v,i)$ atoms and the disjunctive 3D-nCDC constraints in $C$ is a splitting set for $\Pi^{2,a}_{m,n,p}$.
The bottom part is $(\ref{eq:cdc-disj-rel-3d}) \cup (\ref{eq:disj-choose-3d})$ and an answer set $Y_1$ for the bottom part consists of the set $F_V$ of facts $(\ref{eq:cdc-disj-rel-3d})$ that describe disjunctive 3D-nCDC constraints in $C$ and the index of the chosen disjunct from each disjunctive constraint.
An answer set $Y_2$ for the top part $(\ref{eq:rel-disj-3d})$ evaluated with respect to $Y_1$ specifies the chosen basic relation with $\ii{rel}(u,v,r)$ atoms; and $Y_1\cup Y_2$ is an answer set for $\Pi^{2,a}_{m,n,p}$.

The set $F_B \cup F_V$ of facts in $(\ref{eq:rel-3d}) \cup (\ref{eq:cdc-disj-rel-3d})$ represents all 3D-nCDC constraints in the input network $C$.
Hence, an answer set of the subprogram
$\Pi^{2,b}_{m,n,p}$ composed of $\Pi^{2,a}_{m,n,p}$ with the facts in $(\ref{eq:rel-3d})$ and the rule $(\ref{eq:existrel-3d}) $ represents a basic 3D-nCDC network $\hat{C}$ formed by all basic constraints in $C$ and picking precisely one disjunct from each disjunctive constraint in $C$.
$\ii{existrel}(u,v)$ atoms in the answer set of $\Pi^{2,b}_{m,n,p}$ indicate pair of variables for which a constraint exists in $\hat{C}$.

\medskip\noindent{\underline{Rules for MBB and Instantiation of Objects:}}
Next we examine the subprogram $\Pi^{2,c}_{m,n,p}$ formed by combining $\Pi^{2,b}_{m,n,p}$ with the rule $(\ref{eq:generate-infsup-3d})$ and rules analogous to $(\ref{eq:generate-infsup-3d})$ that describe $\ii{inf}_{y}(u, \underline{y})$, $\ii{sup}_{y}(u, \overline{y})$, $\ii{inf}_{z}(u, \underline{z})$, $\ii{sup}_{z}(u, \overline{z})$.
We apply the splitting set theorem to $\Pi^{2,c}_{m,n,p}$:
The set of all possible $\ii{chosen}(u,v,i)$, $\ii{rel}(u,v,r)$, $\ii{disjrel}(u,v,i,r)$, $\ii{existrel}(u,v)$ atoms is a splitting set for $\Pi^{2,c}_{m,n,p}$.
The bottom part of $\Pi^{2,c}_{m,n,p}$ is $\Pi^{2,b}_{m,n,p}$ and the top part is the rule $(\ref{eq:generate-infsup-3d})$ and rules analogous to $(\ref{eq:generate-infsup-3d})$.
An answer set $Y_3$ for the bottom part specifies a basic 3D-nCDC network $\hat{C}$ derived from $C$. An answer set $Y_4$ for the top part evaluated with respect to $Y_3$ describes a possible choice of minimum bounding box of each spatial variable in terms of $\ii{inf}_{x}(u, \underline{x})$, $\ii{sup}_{x}(u, \overline{x})$, $\ii{inf}_{y}(u, \underline{y})$, $\ii{sup}_{y}(u, \overline{y})$, $\ii{inf}_{z}(u, \underline{z})$, $\ii{sup}_{z}(u, \overline{z})$ atoms; and $Y_3 \cup Y_4$ is an answer set for $\Pi^{2,c}_{m,n,p}$.

The rule $(\ref{eq:infsup-ineq-3d})$ and the rules analogous to $(\ref{eq:infsup-ineq-3d})$ for $\ii{inf}_{y}(u, \underline{y})$, $\ii{sup}_{y}(u, \overline{y})$, $\ii{inf}_{z}(u, \underline{z})$, $\ii{sup}_{z}(u, \overline{z})$ ensure the chosen infimum value to be less than or equal to the supremum on each axis.
Proposition~2 of~\cite{Erdogan2004} implies that adding rule $(\ref{eq:infsup-ineq-3d})$ and analogous rules, the answer sets for $\Pi^{2,c}_{m,n,p}$ that do not have a valid minimum bounding box of an object are eliminated.
Thereby, an answer set of subprogram
$\Pi^{2,d}_{m,n,p}$ which is composed of $\Pi^{2,c}_{m,n,p}$, the rule $(\ref{eq:infsup-ineq-3d})$ and analogous rules to $(\ref{eq:infsup-ineq-3d})$ represent a basic 3D-nCDC network $\hat{C}$ derived from $C$ and a valid instantiation of $\ii{inf}_{x}(u, \underline{x})$, $\ii{sup}_{x}(u, \overline{x})$, $\ii{inf}_{y}(u, \underline{y})$, $\ii{sup}_{y}(u, \overline{y})$, $\ii{inf}_{z}(u, \underline{z})$, $\ii{sup}_{z}(u, \overline{z})$ atoms.

Now we examine the subprogram $\Pi^{2,e}_{m,n,p} \eqs \Pi^{2,d}_{m,n,p} \cup (\ref{eq:generate-3d})$.
According to the splitting set theorem, 3D-nCDC constraints in $C$ and the set of all $\ii{chosen}(u,v,i)$, $\ii{rel}(u,v,r)$, $\ii{disjrel}(u,v,i,r)$, $\ii{existrel}(u,v)$, $\ii{inf}_{x}(u, \underline{x})$, $\ii{sup}_{x}(u, \overline{x})$, $\ii{inf}_{y}(u, \underline{y})$, $\ii{sup}_{y}(u, \overline{y})$, $\ii{inf}_{z}(u, \underline{z})$, $\ii{sup}_{z}(u, \overline{z})$ atoms is a splitting set for $\Pi^{2,e}_{m,n,p}$.
The bottom part of $\Pi^{2,e}_{m,n,p}$ is $\Pi^{2,d}_{m,n,p}$ and an answer set $Y_5$ for the bottom part describes a basic 3D-nCDC network $\hat{C}$ derived from $C$ and a valid choice of minimum bounding box for every spatial variable.
An answer set $Y_6$ for the top part $(\ref{eq:generate-3d})$ evaluated with respect to $Y_5$ describes a possible instantiation $A=(a_{i})^{l}_{i=1}$ of objects to variables in $V$
and $Y_5 \cup Y_6$ is an answer set for $\Pi^{2,e}_{m,n,p}$.

In the next step, we add rules of the form $(\ref{eq:project-mp-3d})$ and rules analogous to $(\ref{eq:project-mp-3d})$ for $y,z$ axes into $\Pi^{2,e}_{m,n,p}$ to form subprogram $\Pi^{2,f}_{m,n,p}$.
The set of all $\ii{occ}(u,x,y,z)$, $\ii{chosen}(u,v,i)$, $\ii{rel}(u,v,r)$, $\ii{disjrel}(u,v,i,r)$, $\ii{existrel}(u,v)$, $\ii{inf}_{x}(u, \underline{x})$, $\ii{sup}_{x}(u, \overline{x})$, $\ii{inf}_{y}(u, \underline{y})$, $\ii{sup}_{y}(u, \overline{y})$, $\ii{inf}_{z}(u, \underline{z})$, $\ii{sup}_{z}(u, \overline{z})$ atoms and 3D-nCDC constraints in $C$ is a splitting set for $\Pi^{2,f}_{m,n,p}$.
The bottom part of $\Pi^{2,f}_{m,n,p}$ is $\Pi^{2,e}_{m,n,p}$ and an answer set $Y_7$ for the bottom part describes a basic 3D-nCDC network $\hat{C}$ derived from $C$, a valid choice of minimum bounding box for every variable $u\in V$ and an instantiation of objects to every variable.
The top part of $\Pi^{2,f}_{m,n,p}$ is the rule $(\ref{eq:project-mp-3d})$ and rules analogous to $(\ref{eq:project-mp-3d})$ for $\ii{yocc}(u,y)$, $\ii{zocc}(u,z)$ atoms.
An answer set $Y_8$ for the top part evaluated with respect to $Y_7$ indicates the projection of each generated object over $x,y,z$ axes with $\ii{xocc}(u,x)$, $\ii{yocc}(u,y)$, $\ii{zocc}(u,z)$ atoms; and $Y_7 \cup Y_8$ is an answer set for $\Pi^{2,f}_{m,n,p}$.

The rule $(\ref{eq:cells-inside-mbr-3d})$ and analogous rules for $y,z$ axes insist that cells of every object are generated inside its minimum bounding box. The rule $(\ref{eq:infsup-existcell-3d})$ and analogous rules insist that at least one cell has been generated on the infimum and supremum over every axes to make sure that correct values have been chosen.
Adding rules $(\ref{eq:cells-inside-mbr-3d}),(\ref{eq:infsup-existcell-3d})$ and analogous rules for $y,z$ axes into $\Pi^{2,f}_{m,n,p}$ eliminates answer sets of $\Pi^{2,f}_{m,n,p}$ that do not obey these criteria.
Thus, we form the subprogram $\Pi^{2,g}_{m,n,p}$ which consists of $\Pi^{2,f}_{m,n,p}$, the rules $(\ref{eq:cells-inside-mbr-3d})$, $(\ref{eq:infsup-existcell-3d})$ and rules analogous to $(\ref{eq:cells-inside-mbr-3d})$, $(\ref{eq:infsup-existcell-3d})$ for $\ii{inf}_{y}(u, \underline{y})$, $\ii{sup}_{y}(u, \overline{y})$, $\ii{inf}_{z}(u, \underline{z})$, $\ii{sup}_{z}(u, \overline{z})$.
An answer set $Y_9$ of subprogram $\Pi^{2,g}_{m,n,p}$ represents a basic 3D-nCDC network $\hat{C}$ derived from $C$, a possible instantiation $A=(a_{i})^{l}_{i=1}$ of objects to variables in $V$ and correct bounds $\ii{inf}_{x}(u, \underline{x})$, $\ii{sup}_{x}(u, \overline{x})$, $\ii{inf}_{y}(u, \underline{y})$, $\ii{sup}_{y}(u, \overline{y})$, $\ii{inf}_{z}(u, \underline{z})$, $\ii{sup}_{z}(u, \overline{z})$ for objects.

\medskip\noindent{\underline{Rules for 3D-nCDC Constraints:}}
Rules $(\ref{eq:c1-3d}),(\ref{eq:c2-3d})$ find out whether the instantiation of objects $A$ violates 3D CDC tile constraints (C1), (C2) respectively.
We consider the subprogram $\Pi^{2,h}_{m,n,p} \eqs \Pi^{2,g}_{m,n,p} \cup (\ref{eq:c1-3d}) \cup (\ref{eq:c2-3d})$.
A splitting set for $\Pi^{2,h}_{m,n,p}$ is
the set of all $\ii{occ}(u,x,y,z)$, $\ii{chosen}(u,v,i)$, $\ii{rel}(u,v,r)$, $\ii{disjrel}(u,v,i,r)$, $\ii{existrel}(u,v)$, $\ii{inf}_{x}(u, \underline{x})$, $\ii{sup}_{x}(u, \overline{x})$, $\ii{inf}_{y}(u, \underline{y})$, $\ii{sup}_{y}(u, \overline{y})$, $\ii{inf}_{z}(u, \underline{z})$, $\ii{sup}_{z}(u, \overline{z})$, $\ii{xocc}(u,x)$, $\ii{yocc}(u,y)$, $\ii{zocc}(u,z)$ atoms.
The bottom part is $\Pi^{2,g}_{m,n,p}$ and the top part is $(\ref{eq:c1-3d}) \cup (\ref{eq:c2-3d})$.
An answer set $Y_{10}$ for the top part evaluated with respect to an answer set $Y_{9}$ for the bottom part indicates whether the instantiation $A$ violates conditions (C1), (C2) with $\ii{violated}(u,v)$ atoms and $Y_{9} \cup Y_{10}$ is an answer set for $\Pi^{2,h}_{m,n,p}$.

The rule $(\ref{eq:violated})$ prohibits 3D CDC constraints in (C1), (C2) to be violated for any tile $\ii{rel}(u,v,R)$.
By inserting rule $(\ref{eq:violated})$ into $\Pi^{2,h}_{m,n,p}$, the answer sets of $\Pi^{2,h}_{m,n,p}$ that do not satisfy conditions (C1), (C2) are eliminated.
An answer set of the subprogram $\Pi^{2,i}_{m,n,p} \eqs \Pi^{2,h}_{m,n,p} \cup (\ref{eq:violated})$ represents a basic 3D-nCDC network $\hat{C}$ derived from $C$, a possible instantiation $A=(a_{i})^{l}_{i=1}$ of variables in $V$ that satisfy conditions (C1), (C2) for $\hat{C}$ and the minimum bounding box of the instantiated objects.

Note that $\Pi^{2}_{m,n,p} \eqs \Pi^{2,i}_{m,n,p} \eqs \Pi^{2,h}_{m,n,p} \cup (\ref{eq:violated})$.
If $Z$ is an answer set for $\Pi^{2}_{m,n,p}$, $ \: Z \cap \mathcal{O}_{m,n,p}$ characterize an instantiation $A$ of objects in $D_{m,n,p}$ to variables in $V$ that satisfy basic 3D-nCDC constraints in $\hat{C}$. This means $A$ satisfies basic and disjunctive 3D-nCDC constraints in $C$.
Then, $X$ is a solution of $I_{m,n,p}$ if and only if $X$ can be characterized as $Z \cap \mathcal{O}_{m,n,p}$ for some answer set $Z$ of $\Pi^{2}_{m,n,p}$.

\paragraph{\bf Uniqueness proof}
To prove second part of the theorem, suppose that another answer set $Z'$ for $\Pi^{2}_{m,n,p}$ also characterizes $X$ and $Z'\neq Z$.
By assumption $Z'$ includes the same $\ii{occ}(u,x,y,z)$ atoms as $Z$. Consequently projected coordinates $\ii{xocc}(u,x)$, $\ii{yocc}(u,y)$, $\ii{zocc}(u,z)$ atoms are the same for $Z'$ and $Z$. Minimum bounding box of an object is unique hence $\ii{inf}_{x}(u, \underline{x})$, $\ii{sup}_{x}(u, \overline{x})$, $\ii{inf}_{y}(u, \underline{y})$, $\ii{sup}_{y}(u, \overline{y})$, $\ii{inf}_{z}(u, \underline{z})$, $\ii{sup}_{z}(u, \overline{z})$ atoms are also identical. A pair of objects satisfies only one basic 3D-nCDC relation so the chosen disjuncts from every disjunctive constraints in $C$ must be the same in $Z$ and $Z'$. Consequently, $\ii{chosen}(u,v,i)$ atoms coincide in $Z$ and $Z'$. Since all atoms in the two sets are identical, we obtain $Z'\eqs Z$.




\label{lastpage}
\end{document}